\documentclass[11pt,journal,onecolumn]{IEEEtran}

\usepackage[margin=2.2cm]{geometry}

\usepackage{url}            
\usepackage{booktabs}       
\usepackage{amsfonts}       
\usepackage{nicefrac}       
\usepackage{microtype}      
\usepackage{caption}

\usepackage{array}
\usepackage{makecell}

\usepackage{amsmath,amsfonts,bm}

\def\ceil#1{\lceil #1 \rceil}

\def\1{\bm{1}}

\def\vmu{{\bm{\mu}}}

\def\vm{{\bm{m}}}

\def\vw{{\bm{w}}}

\def\vz{{\bm{z}}}

\def\mA{{\bm{A}}}

\def\mN{{\bm{N}}}

\def\mX{{\bm{X}}}
\def\mY{{\bm{Y}}}

\def\mSigma{{\bm{\Sigma}}}

\DeclareMathAlphabet{\mathsfit}{\encodingdefault}{\sfdefault}{m}{sl}
\SetMathAlphabet{\mathsfit}{bold}{\encodingdefault}{\sfdefault}{bx}{n}

\newcommand{\E}{\mathbb{E}}

\newcommand{\Cov}{\mathrm{Cov}}

\DeclareMathOperator*{\argmax}{arg\,max}
\DeclareMathOperator*{\argmin}{arg\,min}

\usepackage[sort, numbers]{natbib}

\usepackage[utf8]{inputenc} 
\usepackage[T1]{fontenc}    
\usepackage{hyperref}
\hypersetup{colorlinks=true,linkcolor=red,citecolor = blue, linktocpage}

\usepackage{algorithmic}
\usepackage{graphicx}
\usepackage{textcomp}
\usepackage{epstopdf}
\usepackage{times}
\usepackage{amssymb}
\usepackage{amsthm}

\usepackage{xcolor}

\makeatletter
\newtheorem*{rep@theorem}{\rep@title}
\newcommand{\newreptheorem}[2]{%
\newenvironment{rep#1}[1]{%
 \def\rep@title{#2 \ref{##1}}%
 \begin{rep@theorem}}%
 {\end{rep@theorem}}}
\makeatother

\newtheorem{theorem}{Theorem}
\newreptheorem{theorem}{Theorem}

\newtheorem{lemma}{Lemma}
\newreptheorem{lemma}{Lemma}
\newtheorem{assumption}{Assumption}

\newtheorem{proposition}{Proposition}
\newreptheorem{proposition}{Proposition}

\newtheorem{corollary}{Corollary}
\newtheorem{definition}{Definition}

\newtheorem{remark}{Remark}
\usepackage{url}            
\usepackage{booktabs}       
\usepackage{amsfonts}       
\usepackage{nicefrac}       
\usepackage{microtype}      

 \usepackage{cite}
\usepackage{algorithmic}
\usepackage{graphicx}
\usepackage{textcomp}
\usepackage{epstopdf}
\usepackage{times}
\usepackage{amssymb}
\usepackage{amsthm}
\usepackage{graphics}
\usepackage{mathrsfs}
\usepackage{psfrag}
\usepackage{tikz}

\usepackage{xcolor}
\defcitealias{OurGibbsPaper}{Aminian${}^{\star}$ and Bu${}^{\star}$ et al., 2021}
\newcommand{\nn}{\nonumber}

\begin{document}

%

%


\title{Characterizing and Understanding the Generalization Error of Transfer Learning with Gibbs Algorithm}

\author{Yuheng Bu$^*$,
        Gholamali Aminian$^*$,
        Laura Toni,
        Miguel Rodrigues,
        Gregory Wornell,
\thanks{$^*$ Equal Contribution.}
\thanks{Y. Bu and G. Wornell are with the Department of Electrical Engineering and Computer Science, Massachusetts Institute of Technology, Cambridge, MA 02139 (Email: buyuheng, gww@mit.edu).}
\thanks{G. Aminian, L. Toni and M. Rodrigues are with the Electronic and Electrical Engineering Department at University College London, UK, (Email: g.aminian, l.toni, m.rodrigues@ucl.ac.uk).}
}

\allowdisplaybreaks

\maketitle

\begin{abstract}

We provide an information-theoretic analysis of the generalization ability of Gibbs-based transfer learning algorithms by focusing on two popular transfer learning approaches, $\alpha$-weighted-ERM and two-stage-ERM. Our key result is an exact characterization of the generalization behaviour using the conditional symmetrized KL information between the output hypothesis and the target training samples given the source samples. Our results can also be  applied to provide novel distribution-free generalization error upper bounds on these two aforementioned Gibbs algorithms. Our approach is versatile, as it also characterizes the generalization errors and excess risks of these two Gibbs algorithms in the asymptotic regime, where they converge to the $\alpha$-weighted-ERM and two-stage-ERM, respectively. Based on our theoretical results, we show that the benefits of transfer learning can be viewed as a bias-variance trade-off, with the bias induced by the source distribution and the variance induced by the lack of target samples. We believe this viewpoint can guide the choice of transfer learning algorithms in practice.

\end{abstract}

\section{Introduction}

A common assumption in supervised learning is that both the training and test data samples are generated from the same distribution. However, this assumption does not always hold in many applications, as we often have easy access to samples generated from a source distribution, and we want to use the  hypothesis trained using source samples for a different target task, from which only limited data are available. Transfer learning and domain adaptation methods are developed to tackle this problem, and the state of the art transfer learning algorithms based on pre-trained models and fine tuning has led to significant improvements in various applications such as computer vision, natural language processing, etc~\citep{li2012human,long2015learning,yosinski2014transferable,raffel2019exploring}.

Many works try to explain the empirical success of transfer learning from different theoretical perspectives.
The first theoretical analysis for domain adaptation is proposed by \citep{ben2007analysis} for binary classification, where the authors provide a VC-dimension-based excess risk bound for the zero-one loss in terms of $d_\mathcal{A}$-distance as a measure of discrepancy between source and target tasks.\! A new notion of discrepancy measure for transfer learning called transfer-exponent under covariate-shift assumption is proposed in \citep{hanneke2019value}. A minimax lower bound of generalization error for transfer learning in neural networks is derived in \citep{kalan2020minimax}. Recently, an Empirical Risk Minimization (ERM) algorithm via representation learning is proposed in \citep{tripuraneni2020theory}, and an upper bound on the excess risk of the new task is provided in terms of Gaussian complexity. \citep{wang2019transfer} provides an upper bound on excess risk based on instance weighting.\! Using KL divergence as a measure of similarity between source and target data-generating distribution, an information-theoretic generalization error upper bound for transfer learning is proposed in \citep{wu2020information}.


However, these upper bounds on excess risk and generalization error may not entirely capture the generalization ability of a transfer learning algorithm. One apparent reason is the tightness issue, as the proposed bounds~\citep{wang2019transfer} can be loose or even vacuous when evaluated in practice. More importantly, the current definitions of discrepancy metric do not fully characterize all the aspects that could influence the performance of a transfer learning problem, e.g., most discrepancy measures are either algorithm independent (KL divergence in~\citep{wu2020information}), or defined under specific assumption, e.g. transfer-exponent under covariate-shift assumption in~\citep{hanneke2019value}, or only depend on the hypothesis class ($d_\mathcal{A}$-distance in ~\citep{ben2007analysis}), which cannot provide too much insight in selecting different transfer learning algorithms in practice.

To overcome these limitations, we study two Gibbs algorithms, i.e., $\alpha$-weighted Gibbs algorithm and two-stage Gibbs algorithm which can be viewed as randomized version of two ERM-based transfer learning algorithms, i.e.,  $\alpha$-weighted-ERM~\citep{ben2010theory,zhang2012generalization} and two-stage-ERM~\citep{tripuraneni2020theory, donahue2014decaf} using information-theoretic tools.

Our main contributions are as follows:
\begin{itemize}
    \item   We derive exact characterizations of the generalization errors for  $\alpha$-weighted Gibbs algorithm and two-stage Gibbs algorithm using conditional symmetrized KL information. We also provide novel distribution-free upper bounds, which quantify how the number of samples from the source and target will influence the generalization error of these transfer learning algorithms.
     \item We further demonstrate how to use our method to characterize the asymptotic behavior of the generalization error for the Gibbs algorithms under large inverse temperature, where the Gibbs algorithms converge to the $\alpha$-weighted-ERM and Two-stage-ERM, respectively.
     \item By studying the excess risk of the $\alpha$-weighted-ERM and Two-stage-ERM algorithms in the asymptotic regime, we show that the benefits of transfer learning algorithms can be viewed as a bias-variance trade-off, which suggests that the choice of transfer learning algorithm should depend on both the bias induced by the source distribution and the number of target samples.
\end{itemize}

\section{Problem Formulation}
Let $D_s = \{Z_i^s\}_{i=1}^n$ and $D_t = \{Z_j^t\}_{j=1}^m$ be the source and target training sets, respectively, where $Z_i^s$ and $Z_j^t$ are defined on the same alphabet $\mathcal{Z}$.
Note that $D_s$ and $D_t$ are independent, but neither $D_s$ nor $D_t$ is required to be i.i.d generated from the data-generating distribution $P_Z^s$ or $P_Z^t$.
We denote the joint distribution of all source training samples as $P_{D_s}$ and that of the target training samples as $P_{D_t}$. We denote the hypotheses by $w \in \mathcal{W}$, where $\mathcal{W}$ is a hypothesis class. The performance of any hypotheses is measured by a non-negative loss function $\ell:\mathcal{W} \times \mathcal{Z}  \to \mathbb{R}_0^+$, and we can define the  empirical risk and the population risk of a source task as
\begin{align}
    &L_E(w,d_s)\triangleq  \frac{1}{n}\sum_{i=1}^n \ell(w,z_i^s),\\
    &L_P(w,P_{D_s})\triangleq  \mathbb{E}_{P_{D_s}}[L_E(w,D_s)],
\end{align}
and the  empirical risk and the population risk of target task
\begin{align}
    &L_E(w,d_t)\triangleq\frac{1}{m}\sum_{j=1}^m \ell(w,z_j^t),\\
    &L_P(w,P_{D_t})\triangleq  \mathbb{E}_{P_{D_t}}[L_E(w,D_t)].
\end{align}
A transfer learning algorithm can be modeled as a randomized mapping from the source and target training sets $(D_s,D_t)$ onto a hypothesis $W\in\mathcal{W}$  according to the conditional distribution $P_{W|D_s,D_t}$. Thus, the expected transfer generalization error quantifying the degree of over-fitting on the target training data can be written as
\begin{align}\label{Eq: expected GE general}
\overline{\text{gen}}(P_{W|D_s,D_t},P_{D_s},P_{D_t})\triangleq
\mathbb{E}_{P_{W,D_s,D_t}}[ L_P(W,P_{D_t})-L_E(W,{D_t})],
\end{align}
where the expectation is taken over the joint distribution $P_{W,D_s,D_t} =  P_{W|D_s,D_t}\otimes P_{D_s,D_t} $.

\subsection{Two Transfer learning Algorithms}
We focus on the following two transfer learning approaches, including $\alpha$-weighted-ERM and Two-stage-ERM.

\textbf{$\alpha$-Weighted-ERM Transfer Learning:} We denote the hypotheses by $w_\alpha \in \mathcal{W}$ as the output of $\alpha$-weighted-ERM learning algorithm.
The hypothesis $w_\alpha$ is trained by minimizing a convex combination of the source and target task empirical risks as in~\citep{ben2010theory}, i.e.,
\begin{equation}
    L_E(w_\alpha,\!d_s,\!d_t) = (1-\alpha) L_E(w_\alpha,\!d_s)+\alpha L_E(w_\alpha,\!d_t),
\end{equation}
for $0\leq\alpha\leq 1$.


\textbf{Two-stage-ERM Transfer Learning:} Suppose that the hypothesis $w \in \mathcal{W}$ can be written as $w = (w_\phi,w_c)$,
where $w_{\phi} \in \mathcal{W}_\phi$ is the shared hypothesis (parameter) across both source and target tasks, and  $w_c$ denotes some task-specific hypothesis (parameter) for source and target tasks, i.e., $w_c^s \in \mathcal{W}_c$ and $w_c^t \in \mathcal{W}_c$. For example $w_\phi$ collects parameters of first few layers of a neural network for both tasks and $w_c^s$ and $w_c^t$ collect the remaining parameters for source and target tasks respectively. The performance of the pair $(w_\phi,w_c)$ is measured by a non-negative loss function $\ell:\mathcal{W}_c \times \mathcal{W}_\phi \times \mathcal{Z}  \to \mathbb{R}_0^+$. Now, we consider the following two-stage-ERM transfer learning algorithm inspired by  \citep{tripuraneni2020theory}.


\textbf{\textit{First Stage:} }The algorithm first learns the shared hypothesis $w_{\phi}$ and the source-specific hypothesis $w_c^s$ by minimizing the following empirical risk function defined on the source data set at Stage 1:
\begin{equation}
        L_E^{S1}(w_{\phi},w_c^s,d_s)\triangleq\frac{1}{n}\sum_{i=1}^n \ell(w_{\phi},w_c^s,z_i^s).
\end{equation}
\textbf{\textit{Second Stage:}} We fix the shared hypothesis $w_{\phi}$ and learn the target-specific hypothesis $w_c^t$ by minimizing the following empirical risk function defined on the target data set at Stage 2:
    \begin{equation}
        L_E^{S2}(w_{\phi},w_c^t,d_t)=\frac{1}{m}\sum_{j=1}^m \ell(w_{\phi},w_c^t,z_j^t).
    \end{equation}

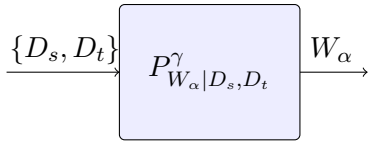
\begin{figure}[t]
\centering
\begin{tikzpicture}[scale=0.6]
 \draw[->] (-3,5.5) -- (-0.5,5.5);
\node at (-1.7,6){$\{D_s,D_t\}$};
\draw[->] (3.5,5.5) -- (5,5.5);
\draw [fill=white!93!blue,rounded corners=2pt](-0.5,4) rectangle (3.5,7);
\node at (1.5,5.5){$P_{W_\alpha|D_s,D_t}^\gamma$};
\node at (4.2,6){$W_\alpha$};
\end{tikzpicture}
\caption{$\alpha$-weighted Gibbs Algorithm}
\label{Fig: Weighted ERM}
\end{figure}

\subsection{Transfer Learning with Gibbs algorithms}
We generalize the ERM-based transfer learning algorithms by considering their Gibbs versions.\! The $(\gamma,\pi(w),f(w,s))$-Gibbs distribution, which was first proposed by \citep{gibbs1902elementary} in statistical mechanics, is defined as:
\begin{equation}\label{Eq: Gibbs Solution}
    P_{{W}|S}^\gamma (w|s) \triangleq \frac{\pi({w}) e^{-\gamma f(w,s)}}{V(s,\gamma)},\quad \gamma\ge 0,
\end{equation}
where $\gamma$ is the inverse temperature, $\pi(w)$ is an arbitrary prior distribution on $W$, $f(w,s)$ is energy function, and $V(s,\gamma) \triangleq \int \pi(w) e^{-\gamma f(w,s)} dw$ is the partition function.

The $(\gamma,\pi(w),L_E(w,d_t))$-Gibbs distribution can be viewed as a randomized version of an ERM algorithm using only target samples if we specify the energy function $f(w,s) = L_E(w,d_t)$. As the inverse temperature $\gamma \to \infty$, the prior distribution $\pi(w)$ becomes negligible, and the Gibbs algorithm converges to the standard supervised-ERM algorithm.

Similarly, we define the following $\alpha$-weighted Gibbs algorithm and two-stage Gibbs algorithm, which can be viewed as randomized $\alpha$-weighted-ERM and randomized two-stage-ERM, respectively.

\textbf{$\alpha$-weighted Gibbs algorithm} generalizes the $\alpha$-weighted-ERM by considering the $(\gamma,\pi(w_\alpha),L_E(w_\alpha,d_s,d_t))$-Gibbs algorithm (see, Figure~\ref{Fig: Weighted ERM})
\begin{align}\label{equ:Gibbs-alpha-weighted}
    P_{W_\alpha|D_s,D_t}^\gamma (w_\alpha|d_s,d_t) = \frac{\pi(w_\alpha) e^{-\gamma L_E(w_\alpha,d_s,d_t)}}{V_\alpha(d_s,d_t,\gamma)}.
\end{align}
The expected transfer generalization error of the $\alpha$-weighted Gibbs algorithm is denoted as
\begin{align}
&\overline{\text{gen}}_{\alpha}(P_{D_s},P_{D_t})\triangleq\overline{\text{gen}}(P_{W_\alpha|D_s,D_t}^\gamma,P_{D_s},P_{D_t}).
\end{align}

\textbf{Two-stage Gibbs algorithm} generalizes the two-stage-ERM by considering the
 $(\gamma, \pi(w_c^t), L_E^{S2}(w_{\phi},w_c^t,d_t))$-Gibbs algorithm algorithm
\begin{align}\label{Eq: two-stage Gibbs algorithm}
    P_{W_c^t|D_t,W_{\phi}}^\gamma (w_c^t|d_t,w_{\phi}) = \frac{\pi(w_c^t) e^{-\gamma L_E^{S_2}(w_\phi,w_c^t,d_t)}}{V_\beta(w_\phi,d_t,\gamma)}
\end{align}
in the second stage, where the learned shared hypothesis $w_\phi$ is the output of the learning algorithm $P_{W_\phi,W_c^s|D_s}$ at the first stage. As shown in Figure~\ref{Fig: two-stage-ERM}, the two-stage Gibbs algorithm is constructed by concatenating two randomized mappings $P_{W_c^t|D_t,W_{\phi}}^\gamma$ and $P_{W_\phi,W_c^s|D_s}$.

The population risk for the target task is defined as:
\begin{equation}
    L_P(w_{\phi},w_c^t,P_{D_t})=\mathbb{E}_{P_{D_t}}[L_E^{S2}(w_{\phi},w_c^t,D_t)],
\end{equation}
and the expected transfer generalization error under two-stage Gibbs algorithm can be denoted as
\begin{align}
\overline{\text{gen}}_{\beta}(P_{D_s},P_{D_t})\triangleq\mathbb{E}_{P_{W_{\phi},W_c^t,D_s,D_t}}[ L_P(W_{\phi},W_c^t,P_{D_t})- L_E^{S2}(W_{\phi},W_c^t,D_t)],
\end{align}
where the expectation is taken over the joint distribution $P_{W_{\phi},W_c^t,D_s,D_t}=P_{W_c^t|D_t,W_{\phi}}^\gamma \otimes P_{D_s,W_\phi} \otimes P_{D_t}$.

\begin{figure}[t]
\centering
\begin{tikzpicture}[scale=0.6]

 \draw[->] (-1.5,5.5) -- (-0.5,5.5);
\node at (-1,6){$D_s$};

\draw[->] (3,5.5) -- (6,5.5);
\draw[-] (3,4.5) -- (4.5,4.5);
\draw[->] (4.5,6.5) -- (6,6.5);

\draw[-] (5.25,5.5) -- (5.25,3.5);
\draw[->] (5.25,3.5) -- (11.5,3.5);
\node at (10.8,3.85){$W_\phi$};

\draw [fill=white!93!blue,rounded corners=2pt](-0.5,4) rectangle (3.5,7);
\node at (1.5,5.5){$P_{W_\phi,W_c^s|D_s}$};


\node at (4.7,6){$W_\phi$};
\node at (4,5){$W_c^s$};
\node at (5.2,7){$D_t$};

\draw [fill=white!93!blue,rounded corners=2pt](6,4) rectangle (10,7);
\node at (8,5.5){$P_{W_c^t|D_t,W_\phi}^\gamma$};
\node at (10.8,6){$W_c^t$};
\draw[->] (10,5.5) -- (11.5,5.5);

\end{tikzpicture}
\caption{Two-stage Gibbs Algorithm}
\label{Fig: two-stage-ERM}
\end{figure}
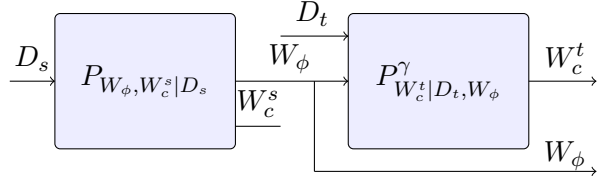

\subsection{Information Measures}
We will be characterizing the aforementioned generalization errors using various information measures. If $P$ and $Q$ are probability measures over space $\mathcal{X}$, and $P$ is absolutely continuous with respect to $Q$, the Kullback-Leibler (KL) divergence between $P$ and $Q$ is given by
$D(P\|Q)\triangleq\int_\mathcal{X}\log\left(\frac{dP}{dQ}\right) dP$. If $Q$ is also absolutely continuous with respect to $P$, the symmetrized KL divergence (a.k.a. Jeffrey's divergence~ \citep{jeffreys1946invariant}) is
\begin{equation}
D_{\mathrm{SKL}}(P\|Q)\triangleq D(P \| Q) + D(Q\|P).
\end{equation}
The mutual information between two random variables $X$ and $Y$ is the KL divergence between the joint distribution and product-of-marginal
distribution $I(X;Y)\triangleq D(P_{X,Y}\|P_X\otimes P_{Y})$, or equivalently, the conditional KL divergence between $P_{Y|X}$ and $P_Y$ averaged over $P_X$, $D(P_{Y|X} \| P_Y|P_{X})\triangleq\int_\mathcal{X}D(P_{Y|X=x} \| P_Y) dP_{X}(x)$. By swapping the role of $P_{X,Y}$ and $P_X\otimes P_{Y}$ in mutual information, we get the lautum information introduced by \citep{palomar2008lautum}, $L(X;Y)\triangleq D(P_X\otimes P_{Y}\| P_{X,Y})$. Finally, the symmetrized KL information between $X$ and $Y$ is given by \citep{aminian2015capacity}:
\begin{align}
   &I_{\mathrm{SKL}}(X;Y)\triangleq  D_{\mathrm{SKL}}(P_{X,Y}\|P_X\otimes P_Y)=
   I(X;Y)+ L(X;Y).
\end{align}
Throughout the paper, upper-case letters denote random variables (e.g., $Z$), lower-case letters denote the realizations of random variables (e.g., $z$), and calligraphic letters denote sets (e.g., $\mathcal{Z}$).
All the logarithms are the natural ones, and all information measure units are in nats. $\mathcal{N}(\!\vmu,\Sigma\!)$ denotes a Gaussian distribution with mean $\vmu$ and covariance matrix \!$\Sigma$.

\section{Related Work}

\textbf{Other Interpretations for Gibbs Algorithm:} Besides viewing the Gibbs algorithm as randomized ERM, there are additional interpretations for considering Gibbs algorithm in transfer learning.   


\textbf{SGLD:} The Stochastic Gradient Langevin Dynamics (SGLD), which can be viewed as noisy version of Stochastic Gradient Descent (SGD), is defined as:
\begin{equation*}
        W_{k+1}=W_k-\eta\nabla L_E(W_k,d_t)+\sqrt{\frac{2\beta}{\gamma}}\zeta_k, \quad k=0,1,\cdots,
\end{equation*}
where $\zeta_k$ is a standard Gaussian random vector and $\eta>0$ is the step size. In \citep{raginsky2017non}, it is proved that under some conditions on the loss function, the conditional distribution $P_{W_k|D_t}$ induced by SGLD algorithm is close to $(\gamma,\pi(W_0),L_E(w_k,d_t))$-Gibbs distribution in 2-Wasserstein distance for sufficiently large $k$.


\textbf{Information Risk Minimization:} The Gibbs algorithm also arises when conditional KL divergence is used as a regularizer to penalize over-fitting in the information risk minimization framework.
It is shown in \citep{xu2017information,zhang2006information,zhang2006E} that the solution to the regularized ERM problem
\begin{align} 
    P^{\star}_{W|D_t}=\arg \inf_{P_{W|D_t}}\big( &\mathbb{E}_{P_{W,D_t}}[L_E(W,D_t)]  +\frac{1}{\gamma} D(P_{W|D_t}\|\pi(W)|P_{D_t})\big),
\end{align}
corresponds to the $(\gamma,\pi(w),L_E(w,d_t))$-Gibbs distribution. The inverse temperature $\gamma$ controls the regularization term and balances between over-fitting and generalization.

\textbf{Generalization Error of the Gibbs Algorithm:}
An exact characterization of the generalization error for Gibbs algorithm in terms of symmetrized KL information is provided by \citepalias{OurGibbsPaper}. 
The authors also provide a generalization error upper bound with the rate of $\mathcal{O}\left(\alpha/n\right)$ under the sub-Gaussian assumption. An information-theoretic upper bound with similar rate $\mathcal{O}\left(\alpha/n\right)$ is provided in \citep{raginsky2016information} for the Gibbs algorithm with bounded loss function, and PAC-Bayesian bounds using a variational approximation of Gibbs posteriors are studied in \citep{alquier2016properties}. \citep{asadi2020chaining,kuzborskij2019distribution} both focus on bounding the excess risk of the Gibbs algorithm.



\textbf{Information-theoretic generalization error bounds for Supervised Learning:} Recently, \citep{russo2019much,xu2017information} propose to use the mutual information between the input training set and the output hypothesis to upper bound the expected generalization error.  Multiple approaches have been proposed to tighten these mutual information-based bound. \citep{bu2020tightening} provides tighter bounds by considering the individual sample mutual information, \citep{asadi2018chaining,asadi2020chaining} propose using chaining mutual information,  and \citep{steinke2020reasoning,hafez2020conditioning,haghifam2020sharpened} advocate the conditioning and processing techniques.
Information-theoretic generalization error bounds using other information quantities are also studied, such as, $f$-divergence~\citep{jiao2017dependence}, $\alpha$-R\'enyi divergence and maximal leakage~\citep{issa2019strengthened,esposito2019generalization}, and Jensen-Shannon divergence~\citep{aminian2020jensen,aminian2021information}. Using rate-distortion theory, \citep{masiha2021learning,bu2020information,bu2021population} provide information-theoretic generalization error upper bounds for model misspecification and model compression.


\textbf{Other Analyses of Transfer Learning:}
In hypothesis transfer learning problem \citep{kuzborskij2013stability}, where we only have access to the learned source hypotheses instead of the source training data, an upper bound on the leave-one-out error measured by square loss is provided. An extension of hypothesis transfer learning is studied in \citep{kuzborskij2017fast}, where an algorithm combining the hypotheses from multiple sources based on regularized ERM principle is studied.
There are also works focusing on the theoretical aspects of domain adaptation, see~\citep{ben2007analysis,ben2010theory,mansour2009domain,mansour2009multiple,germain2016new,david2010impossibility}, which are also related to our problem. Note that in domain adaptation, there is no labeled target data and only unlabeled target samples are available. Actually, having access to target labeled data would improve the performance of the learning algorithm for target task~\citep{mansour2021theory,wang2019transfer}.



Note that we provide an \emph{exact} characterization of the generalization error for Gibbs algorithms in transfer learning scenarios, which differs from this body of research.

\section{Generalization Error of Transfer Learning Algorithm}
We now offer an exact characterizations of the expected transfer generalization errors in terms of symmetrized KL information for the $\alpha$-weighted and two-stage Gibbs algorithms, respectively. Then, combining the exact characterization of expected transfer generalization error for Gibbs algorithms with a conditional mutual information-based generalization error upper bound, we derive novel distribution-free upper bounds for these two Gibbs algorithms. Finally, we provide another exact characterizations of the generalization errors in terms of symmetrized KL divergence, which is shown to be useful in the asymptotic analysis.

\subsection{Exact Characterization of Generalization Error Using Symmetrized KL Information}\label{Sec: weighted-ERM}

One of our main results, which characterizes the exact expected transfer generalization error of the $\alpha$-weighted Gibbs algorithm with prior distribution $\pi(w_\alpha)$, is as follows:
\begin{theorem}[Proved in Appendix~\ref{app: exact Characterization}]\label{Theorem: Gibbs alpha Transfer result}
For the $\alpha$-weighted Gibbs algorithm, $0<\alpha<1$ and $\gamma>0$,
\begin{equation}\label{equ:GibbsposteriorTransfer}
    P_{W_\alpha|D_s,D_t}^\gamma (w_\alpha|d_s,d_t) = \frac{\pi(w_\alpha) e^{-\gamma L_E(w_\alpha,d_s,d_t)}}{V_\alpha(d_s,d_t,\gamma)},
\end{equation}
the expected transfer generalization error is given by
\begin{equation}
   \overline{\text{gen}}_{\alpha}(P_{D_s},P_{D_t}) =  \frac{I_{\mathrm{SKL}}(W_\alpha;D_t|D_s)}{\gamma\alpha}.
\end{equation}
\end{theorem}

We also provide an exact characterization of the expected transfer generalization error for two-stage Gibbs algorithm using conditional symmetrized KL information.

\begin{theorem}[Proved in Appendix~\ref{app: exact Characterization}]\label{Theorem: Two-stage-ERM approach}
The expected transfer generalization error of the  two-stage Gibbs algorithm in ~\eqref{Eq: two-stage Gibbs algorithm} is given by
\begin{align}
   &\overline{\text{gen}}_{\beta}(P_{D_s},P_{D_t})=\frac{ I_{\mathrm{SKL}}(D_t;W_c^t|W_\phi)}{\gamma}.
\end{align}
\end{theorem}

To the best of our knowledge, these results are the first exact characterizations of the expected transfer generalization error for the $\alpha$-weighted and two-stage Gibbs algorithm. Note that both Theorem~\ref{Theorem: Gibbs alpha Transfer result} and Theorem~\ref{Theorem: Two-stage-ERM approach} only assume that the loss function is non-negative and the training set of source and target are independent, and they hold even for non-i.i.d training samples in source and target training sets.

The expected transfer generalization errors are non-negative, i.e., $\overline{\text{gen}}_{\alpha}(P_{D_s},P_{D_t})\ge 0$ and $\overline{\text{gen}}_{\beta}(P_{D_s},P_{D_t})\ge 0$, which follows by the non-negativity of the conditional symmetrized KL information.



\subsection{Example: Mean Estimation}\label{sec: example}
We now consider a simple mean estimation problem, where the symmetrized KL information can be computed exactly, to demonstrate the usefulness of our Theorems. All details are provided in Appendix~\ref{app: Mean Estimation}.

Consider the problem of learning the mean $\vmu_t \in \mathbb{R}^d$ of the target task using $n$ i.i.d. source samples $D_s=\{Z^s_i\}_{i=1}^{n}$ and $m$ i.i.d. target samples  $D_t=\{Z^t_j\}_{j=1}^{m}$.
We assume that the samples
from the source and target tasks satisfying $\mathbb{E}[Z^s]=\vmu_s$, $\mathrm{cov}[Z^s] = \sigma_s^2 I_d$ and $\mathbb{E}[Z^t]=\vmu_t$, $\mathrm{cov}[Z^t] = \sigma_t^2 I_d$, respectively.
We adopt the mean-squared loss $\ell(\vw,\vz) = \|\vz-\vw\|_2^2$, and assume a Gaussian prior for the mean $\pi (\vw) =  \mathcal{N}(\vmu_0,\sigma^2_0 I_d)$.

For the $\alpha$-weighted Gibbs algorithm, if we set inverse-temperature $\gamma=\frac{m+n}{2\sigma^2}$ and $\alpha = \frac{m}{m+n}$, then the $(\frac{m+n}{2\sigma^2},\mathcal{N}(\vmu_0,\sigma^2_0 I_d),L_E(\vw_\alpha,d_s,d_t))$-Gibbs algorithm is given by the following posterior \citep{murphy2007conjugate},
\begin{align}\label{equ:mean_alg}
    P_{W_\alpha|D_t,D_s}^{\gamma}(\vw_\alpha|D_s,D_t)\sim \mathcal{N}\Big( \vm_\alpha , \sigma_1^2 I_d\Big),
\end{align}
with $\vm_\alpha = \frac{\sigma_1^2 }{\sigma_0^2}\vmu_0 +\frac{\sigma_1^2 }{\sigma^2}\big(\sum_{i=1}^{ n} Z^s_i +\sum_{j=1}^{m} Z^t_j \big)$, and $\sigma_1^2=\frac{\sigma_0^2 \sigma^2}{(m+n)\sigma_0^2 +\sigma^2}$. Since $P_{W_\alpha|D_s,D_t}^{\gamma}$ is Gaussian, the conditional symmetrized KL information does not depend on the distribution $P_{Z^t}$ when $\mathrm{cov}[Z^t] = \sigma_t^2 I_d$, i.e.,
\begin{align}
    I_{\mathrm{SKL}}(W_\alpha;D_t|D_s) &= \frac{md \sigma_0^2 \sigma_t^2}{((m+n)\sigma_0^2 + \sigma^2)\sigma^2}.
\end{align}
From Theorem~\ref{Theorem: Gibbs alpha Transfer result}, the expected transfer generalization error of this algorithm can be computed exactly as:
\begin{align}\label{equ:alpha_Gaussian}
  \overline{\text{gen}}_{\alpha}(P_{D_s},P_{D_t}) &= \frac{I_{\mathrm{SKL}}(W_\alpha;D_t|D_s)}{\gamma \alpha}
  = \frac{2d \sigma_0^2 \sigma_t^2}{(m+n)(\sigma_0^2 + \frac{1}{2\gamma})}.
\end{align}
For the two-stage Gibbs algorithm, we learn the first $d_\phi$ components $\vmu_\phi \in \mathbb{R}^{d_\phi}$ using source samples, and use the $(\frac{m}{2\sigma^2}, \mathcal{N}(\vmu_{0,c},\sigma^2_0 I_{d_c}), L_E^{S2}(\vmu_{\phi},\vw_c^t,d_t))$-Gibbs algorithm to learn the remain $d_c=d-d_\phi$ components. Following similar steps, by Theorem~\ref{Theorem: Two-stage-ERM approach}, we have
\begin{align}\label{equ:beta_Gaussian}
 \overline{\text{gen}}_{\beta}(P_{D_s},P_{D_t}) &= \frac{I_{\mathrm{SKL}}(W_c^t;D_t|W_{\phi})}{\gamma }= \frac{2d_c \sigma_0^2 \sigma_t^2}{m(\sigma_0^2 +\frac{1}{2\gamma}) }.
\end{align}

\begin{remark}[Comparison with Supervised Learning]
It is shown in \citepalias{OurGibbsPaper} that the generalization error of a supervised Gibbs algorithm is
\begin{equation}\label{Eq: supervised gen}
  \overline{\text{gen}}(P_{W|D_t}^{\gamma},P_{D_t})=\frac{2 d \sigma_0^2 \sigma_t^2 }{m(\sigma_0^2 +\frac{1}{2\gamma} )},
\end{equation}
where $P_{W|D_t}^{\gamma}$ is $(\frac{m}{2\sigma^2}, \mathcal{N}(\vmu_{0},\sigma^2_0 I_{d}),L_E(w,d_t))$-Gibbs algorithm. Comparing to the supervised learning algorithm, the $\alpha$-weighted Gibbs algorithm reduces the generalization error to $\mathcal{O}(\frac{d}{m+n})$ by fitting $n$ source samples and $m$ target samples simultaneously, and the two-stage Gibbs algorithm achieves the rate of $\mathcal{O}(\frac{d_c}{m})$ by only learning $\vw_c^t \in \mathbb{R}^{d_c}$ from the target samples $D_t$, and learning $\vw_\phi$ from $D_s$.
\end{remark}

\begin{remark}[Effect of Source samples] As shown in \eqref{equ:alpha_Gaussian} and \eqref{equ:beta_Gaussian}, the transfer generalization errors of this mean estimation problem do not depend on the distribution of sources samples $D_s$. The reason is that the effect of sources samples is cancelled out in generalization error by subtracting the empirical risk from the population risk. Although different sources samples (distribution) do not change generalization error, they will influence the population risks and excess risks, and more detailed discussion is provided in Appendix~\ref{app: Mean Estimation}.
\end{remark}


\subsection{Distribution-free Upper Bounds}\label{sec: Distribution-free Upper Bound}
To understand the behaviour of expected transfer generalization error, we provide distribution-free upper bounds in this subsection. These bounds quantify how the generalization errors of $\alpha$-weighted and two-stage Gibbs algorithms depend on the number of target  (source) samples $m$ ($n$), and can be applied when directly computing symmetrized KL information is hard.

We first provide a conditional mutual information based upper bound on the expected transfer generalization error for any general learning algorithm $P_{W|D_s,D_t}$ under i.i.d and $\sigma$-sub-Gaussian assumption.
\begin{theorem}[Proved in Appendix~\ref{app: proof Theorem: General result}]\label{Theorem: General result}
Suppose that the target training samples $D_t=\{Z_j^t\}_{j=1}^m$ are i.i.d generated from the distribution $P_Z^t$, and the non-negative loss function $\ell(w,Z)$ is $\sigma$-sub-Gaussian\footnote{A random variable $X$ is $\sigma$-sub-Gaussian if $\log \mathbb{E}[e^{\lambda(X-\mathbb{E}X)}] \le \frac{\sigma^2\lambda^2}{2}$, $\forall \lambda \in \mathbb{R} $.}
under the distribution $P_Z^t \otimes P_W$. Then the following upper bound holds
\begin{align}\label{eq: General upper bound transfer}
   &|\overline{\text{gen}}(P_{W|D_s,D_t},P_{D_s},P_{D_t})|
\leq
   \sqrt{\frac{2 \sigma^2}{m} I(W;\!D_t|\!D_s)}.
\end{align}

\end{theorem}

The following distribution-free upper bound on the expected transfer generalization error for $\alpha$-weighted Gibbs algorithm can be obtained by combining the upper bound in Theorem~\ref{Theorem: General result} and the exact characterization in Theorem~\ref{Theorem: Gibbs alpha Transfer result}.

\begin{theorem}[Proved in Appendix~\ref{app: Distribution-free Upper Bound}]\label{Theorem:  distribution-free upper bound weighted}
Suppose that the target training samples $D_t=\{Z_j^t\}_{j=1}^m$ are i.i.d generated from the distribution $P_Z^t$, and the non-negative loss function $\ell(w,z)$ is $\sigma_\alpha$-sub-Gaussian
under the distribution $P_Z^t \otimes P_{W_\alpha}$.
If we further assume $C_\alpha\le \frac{L(W_\alpha;D_t|D_s)}{I(W_\alpha;D_t|D_s)}$ for some $C_\alpha \ge 0$, then for the $\alpha$-weighted Gibbs algorithm and $0<\alpha<1$,
\begin{align}
    &\overline{\text{gen}}_{\alpha}(P_{D_s},P_{D_t}) \leq \frac{2\sigma_\alpha^2\gamma\alpha}{(1+C_\alpha)m}.
\end{align}
\end{theorem}
\begin{remark}
Let $\alpha=\frac{m}{n+m}$, then we have
\begin{align}
&\overline{\text{gen}}_{\alpha}(P_{D_s},P_{D_t}) \leq \frac{2\sigma_\alpha^2\gamma}{(1+C_\alpha)(n+m)},
\end{align}
which is lower than the distribution-free upper bound for $(\gamma,\pi(w),L_E(w,d_t))$-Gibbs algorithm $P_{W|D_t}^\gamma$ provided in \citepalias[Theorem~2]{OurGibbsPaper}, i.e.,
$\overline{\text{gen}}_{\alpha}(P_{W|D_t}^\gamma,\!P_{D_t})\!\leq\!\frac{2\sigma^2\gamma}{(1+C_E)m}$, if $C_E\!=\!C_\alpha$ and $\sigma^2\!=\!\sigma_\alpha^2$.
\end{remark}
Using similar approach, we can obtain a distribution-free upper bound on the expected transfer generalization error for the two-stage Gibbs algorithm.
\begin{theorem}[Proved in Appendix ~\ref{app: Distribution-free Upper Bound}]\label{Theorem: distribution-free upper bound two-stage}
Suppose that the target training samples $D_t=\{Z_j^t\}_{j=1}^m$ are i.i.d generated from the distribution $P_Z^t$, and the non-negative loss function $\ell(w,z)$ is $\sigma_\beta$-sub-Gaussian
under distribution $P_Z^t\otimes P_{W_c^t|W_\phi=w_\phi}$ for all $w_\phi \in \mathcal{W}_\phi$.
If we further assume $C_\beta\le \frac{L(W_c^t;D_t|W_\phi)}{I(W_c^t;D_t|W_\phi)}$ for some $C_\beta \ge 0$, then for the two-stage Gibbs algorithm in~\eqref{Eq: two-stage Gibbs algorithm}, we have
\begin{align}
   \overline{\text{gen}}_{\beta}(P_{D_s},P_{D_t})\leq \frac{2\sigma_\beta^2\gamma}{(1+C_\beta)m}.
\end{align}
\end{theorem}
\begin{remark}[Choice of $C_\beta$ and $C_\alpha$]
Setting $C_\alpha=0$ in Theorem~\ref{Theorem:  distribution-free upper bound weighted} and $C_\beta=0$ in Theorem~\ref{Theorem: distribution-free upper bound two-stage} is always valid since the lautum information is always positive whenever the mutual information is positive.
\end{remark}

\subsection{Exact Characterization of Generalization Error Using Symmetrized KL divergence}

In this section, we provide exact characterizations of expected transfer generalization errors for $\alpha$-weighted and two-stage Gibbs algorithms using conditional symmetrized KL divergence by considering the Gibbs algorithm defined with the population risks. Such a result is very useful in the asymptotic analysis~Section~\ref{sec:limit_gen}.

\begin{theorem}[Proved in Appendix~\ref{App: new rep skl divergence}] \label{Theorem: rep in skl divergence transfer} The expected transfer generalization error of the $\alpha$-weighted Gibbs algorithm in~\eqref{equ:Gibbs-alpha-weighted} is given by:
\begin{align}
    & \overline{\text{gen}}_{\alpha}(P_{D_s},P_{D_t}) =\frac{D_{\mathrm{SKL}}(P_{W_\alpha|D_s,D_t}^{\gamma}\|P_{W_\alpha|D_s}^{\gamma,L_\alpha(w_\alpha,d_s,P_{D_t})}|P_{D_s} P_{D_t})}{\gamma \alpha},
\end{align}
where $P_{W_\alpha|D_s}^{\gamma,L_\alpha(w_\alpha,d_s,P_{D_t})}$ is  $(\gamma,\pi(w_\alpha),L_\alpha(w_\alpha,d_s,P_{D_t}))$-Gibbs algorithm with  $L_\alpha(w,d_s,P_{D_t})\triangleq \alpha L_P(w_\alpha,P_{D_t})+(1-\alpha)L_E(w_\alpha,d_s)$.
\end{theorem}



 Similar result can be obtained for the two-stage Gibbs algorithm.

\begin{theorem}[Proved in Appendix~\ref{App: new rep skl divergence}]\label{Theorem: rep in skl divergence transfer two-stage}
The expected transfer generalization error of the two-stage Gibbs algorithm in~\eqref{Eq: two-stage Gibbs algorithm} is given by:
\begin{align}
   &\overline{\text{gen}}_{\beta}(P_{D_s},P_{D_t})= \frac{D_{\mathrm{SKL}}(P^\gamma_{W_c^t|D_t,W_\phi}\| P_{W_c^t|W_\phi}^{\gamma,L_P(w_\phi,w_c^t,P_{D_t})}|P_{D_t}P_{W_\phi})}{\gamma},
\end{align}
where $P_{W_c^t|W_\phi}^{\gamma,L_P(w_\phi,w_c^t,P_{D_t})}$ is the $(\gamma,\!\pi(w_c^t),\!L_P(w_\phi,\!w_c^t,\!P_{D_t}))$-Gibbs algorithm.
\end{theorem}

More discussions about the connection between the results obtained using symmetrized KL information and those of symmetrized KL divergence is provided in Appendix~\ref{App: new rep skl divergence}.

\section{Asymptotic Behavior of Generalization Error and Excess Risk}\label{sec:asymptotic}

In this section, we first consider the asymptotic behavior of the generalization error for the two Gibbs algorithms as the inverse temperature $\gamma \to \infty$. Note that in this regime, both Gibbs algorithms converge to the corresponding ERM algorithms, and the distribution-free upper bounds obtained in the previous section would become vacuous. Then, we show that such results can be applied to characterize the excess risks of the two ERM algorithms as $m,n \to \infty$, which provides some intuitions for the selection of different transfer learning algorithms.



\subsection{Generalization Error}\label{sec:limit_gen}
\textbf{$\alpha$-weighted-ERM:}
We assume that there exists a unique $\hat{W}_{\alpha}(D_s,D_t)$ and a unique $\hat{W}_{\alpha}(D_s)$ that minimizes the risk $L_E(w,D_s,D_t)$ and $L_\alpha(w,D_s,P_{D_t})$, respectively, i.e.,
\begin{align}
    \hat{W}_{\alpha}(D_s,D_t)& = \argmin_{w\in \mathcal{W}} L_E(w,D_s,D_t), \\
    \hat{W}_{\alpha}(D_s)& = \argmin_{w\in \mathcal{W}} L_\alpha(w,D_s,P_{D_t}).
\end{align}
It is shown in \citep{hwang1980laplace} that if the following Hessian matrices
\begin{align}
H^*(D_s,D_t) &\triangleq \nabla^2_w L_E(w,D_s,D_t)\big|_{w = \hat{W}_\alpha(D_s,D_t)},\\
H^*(D_s) &\triangleq \nabla^2_w L_\alpha(w,D_s,P_{D_t})\big|_{w = \hat{W}_\alpha(D_s)}
\end{align}
are not singular, then, as $\gamma \to \infty$
\begin{align}\label{equ:Gaussian_approx}
    P_{W_\alpha|D_s,D_t}^\gamma &\to \mathcal{N}(\hat{W}_\alpha(D_s,D_t), \frac{1}{\gamma}H^*(D_s,D_t)^{-1}),
    \text{ and } P_{W_\alpha|D_s}^\gamma \to \mathcal{N}(\hat{W}_\alpha(D_s), \frac{1}{\gamma}H^*(D_s)^{-1})
\end{align}
in distribution. Thus, the conditional symmetrized KL divergence in Proposition~\ref{Theorem: rep in skl divergence transfer} can be evaluated directly using Gaussian approximations.

\begin{proposition}[Proved in Appendix~\ref{app:gen_limit}]\label{Prop: large gamma alpha}
If the Hessian matrices $H^*(D_s,D_t) = H^*(D_s) = H^*$ are independent of $D_s$ and $D_t$, then the generalization error of the $\alpha$-weighted-ERM algorithm is
\begin{equation*}
     \overline{\text{gen}}_{\alpha}(P_{D_t},\!P_{D_s}) =\frac{\mathbb{E}_{P_{D_s\!,D_t}}[\|\hat{W}_\alpha(\!D_s,\!D_t\!)-\hat{W}_\alpha(\!D_s\!)\|^2_{H^*}]}{\alpha},
\end{equation*}
where the notation $\|W\|_H^2 \triangleq W^\top H W$.
\end{proposition}


We can use Proposition~\ref{Prop: large gamma alpha} to obtain the generalization error of the maximum likelihood estimates (MLE) in the asymptotic regime $m, n\to \infty$. More specifically, suppose that we have $m$ and $n$ i.i.d. samples generated from the target distribution $P_Z^t$ and source distribution $P_Z^s$, respectively. We want to fit the training data with a parametric distribution family $\{f(z|\vw_\alpha)\}$ using the $\alpha$-weighted-ERM algorithm, where $\vw_\alpha\in \mathcal{W}\subset \mathbb{R}^d$ denotes the parameter. Here, the true data-generating distribution may not belong to the parametric family, i.e., $P_Z^s, P_Z^t \notin \{f(\cdot|\vw_\alpha)|\vw_\alpha \in \mathcal{W} \}$.

If we use the log-loss $\ell(\vw_\alpha,z)=-\log f(z|\vw_\alpha)$ in the $\alpha$-weighted Gibbs algorithm, and set $\alpha=\frac{m}{m+n}$, as $\gamma \to \infty$, it converges to the $\alpha$-weighted-ERM algorithm, which is equivalent to the following MLE, i.e.,
\begin{align}
   & \hat{W}_\alpha(D_s,D_t)\quad = \argmax_{\vw_\alpha\in \mathcal{W}} \sum_{i=1}^n \log f(Z_i^s|\vw_\alpha)+\sum_{j=1}^m \log f(Z_j^t|\vw_\alpha).
\end{align}
If we further let $m,n\to \infty$, under regularization conditions for MLE (details in Appendix~\ref{app:MLE}) which guarantee that $\hat{W}_\alpha(D_s,D_t)$ and $\hat{W}_\alpha(D_s)$ are unique, we can show that
\begin{equation}
     \hat{W}_\alpha(D_s,D_t)-\hat{W}_\alpha(D_s) \to \mathcal{N}\big(0,  \frac{m}{(m+n)^2}\bar{J}(\vw^*_\alpha)^{-1}\mathcal{I}_t(\vw^*_\alpha) \bar{J}(\vw^*_\alpha)^{-1}\big).
 \end{equation}
where
\begin{equation}
    \vw^*_{\alpha}\triangleq \argmin_{\vw\in \mathcal{W}} n D(P_Z^s\|f(\cdot|\vw))+m D(P_Z^t\|f(\cdot|\vw)),
\end{equation}
$\bar{J}(\vw^*_{\alpha})$ is the weighted expectation of the Hessian matrix, and $\bar{\mathcal{I}}(\vw^*_{\alpha})$ is the weighted Fisher information matrix. Detailed definitions of $\bar{J}$ and $\bar{\mathcal{I}}$ and proofs are provided in Appendix~\ref{app:MLE_limit}.

In addition, the Hessian matrix $H^*(D_s,D_t) \to \bar{J}(\vw^*_{\alpha})$ as $m, n\to \infty$, which is independent of the samples $D_s,D_t$. Thus, Proposition~\ref{Prop: large gamma alpha} gives
\begin{equation}\label{equ:gen_alpha_mle}
    \overline{\text{gen}}_{\alpha}(P_{D_t},P_{D_s})  = \frac{\mathrm{tr}( \bar{\mathcal{I}}(\vw^*_{\alpha})\bar{J}(\vw^*_{\alpha})^{-1})}{n+m},
\end{equation}
which scales as $\mathcal{O}(\frac{d}{m+n})$.

\begin{table*}[t]
  \caption{Comparison of different algorithms under MLE setting.}
  \label{tab:comparison}
  \centering
	\begin{tabular}{cccc}
	\\
    	\toprule
    	 & \textbf{Standard ERM} & \textbf{$\alpha$-weighted-ERM}  & \textbf{Two-stage-ERM} \\
    	\midrule

    	Excess risk bias&$0$ & $\|\vw^*_{\alpha}-\vw_t^* \|^2_{J_t(\vw_t^*)}$ & $\big\|[\vw^{s*}_{\phi},\vw^{st*}_{c}] - [\vw^{t*}_{\phi}, \vw^{t*}_{c}]\big\|^2_{J_t(\vw^{t*}_{\phi}, \vw^{t*}_{c})}$ \\

    	Excess risk variance&$\mathcal{O}(\frac{d}{m})$ & $\mathcal{O}(\frac{d}{m+n})$ & $\mathcal{O}(\frac{d}{n}+\frac{d_c}{m})$ \\
    	
    	Generalization error  &$\mathcal{O}(\frac{d}{m})$ & $\mathcal{O}(\frac{d}{m+n})$ & $\mathcal{O}(\frac{d_c}{m})$ \\
    	\bottomrule
  \end{tabular}
\end{table*}

\textbf{Two-stage-ERM:}
We assume that there exists one unique $ \hat{W}^{t}_{c}(D_t, W_{\phi})$ which minimize the empirical risk of stage 2,
 \begin{align}
   \hat{W}^{t}_{c}(D_t, W_{\phi}) &\triangleq \argmin_{w_c\in \mathcal{W}_c} L_E^{S2}(W_{\phi},w_c,D_t),
\end{align}
and there is one unique $\hat{W}^{t}_{c}(W_{\phi})$ which minimize the population risk by considering a fixed $W_\phi$,
\begin{align}
   \hat{W}^{t}_{c}(W_{\phi}) &\triangleq \argmin_{w_c\in \mathcal{W}_c} L_P(W_{\phi},w_c,P_{D_t}).
 \end{align}
Similarly, if the following Hessian matrices
\begin{align}
H_c^*(D_t,\!W_{\phi}) &\triangleq \nabla^2_{w_c}\! L_E^{S2}(W_{\phi},\!w_c,\!D_t)\big|_{w_c = \hat{W}^{t}_{c}(\!D_t,\! W_{\phi}\!)}\\
H_c^*(W_{\phi}) &\triangleq \nabla^2_{w_c}\! L_P(W_{\phi},w_c,P_{D_t})\big|_{w_c = \hat{W}^{t}_{c}( W_{\phi})}
\end{align}
are not singular, we can obtain the following result by evaluating the conditional symmetrized KL divergence in Proposition~\ref{Theorem: rep in skl divergence transfer two-stage} using similar Gaussian approximation as in \eqref{equ:Gaussian_approx}.
\begin{proposition}[Proved in Appendix~\ref{app:gen_limit}]\label{Prop: large gamma beta}
If Hessian matrices $H_c^*(D_t, W_{\phi}) = H_c^*(W_{\phi}) = H_c^*$ are independent of $D_s, D_t$, then the generalization error of the two-stage-ERM algorithm is
\begin{align*}
    & \overline{\text{gen}}_{\beta}(P_{D_t},P_{D_s})  =\mathbb{E}_{D_s,D_t,W_\phi}[\|\hat{W}^{t}_{c}(D_t, W_{\phi})-\hat{W}^{t}_{c}(W_{\phi})\|^2_{H_c^*}].
\end{align*}
\end{proposition}

Consider a similar MLE setting as we did for the $\alpha$-weighted-ERM algorithm, and now we want to fit data with a parametric distribution family $\{f(z_j^t|\vw_\phi,\vw_c^t)\}_{j=1}^m$ using the two-stage-ERM algorithm, where $\vw_\phi \in \mathcal{W}_\phi\subset \mathbb{R}^{d_\phi},\vw_c^t\in \mathcal{W}_c\subset \mathbb{R}^{d_c}$ denote the shared and specific parameters, respectively. 

If we use the log-loss $\ell(\vw_\phi,\vw_c^t,z)=-\log f(z|\vw_\phi,\vw_c^t)$ in the two-stage Gibbs algorithm, as $\gamma \to \infty$, it converges to the following two-stage MLE approach,
\begin{align*}
    [\hat{W}_{\phi}(D_s), \hat{W}^{s}_{c}(D_s)] &\triangleq \argmax_{[\vw_\phi,\vw_c]\in \mathcal{W}} \sum_{i=1}^n \log f(Z_i^s|\vw_\phi,\vw_c), \\
    \hat{W}^{t}_{c}(D_t, \hat{W}_{\phi}) &\triangleq \argmax_{\vw_c\in \mathcal{W}_c} \sum_{j=1}^m \log f(Z_j^t|\hat{W}_{\phi},\vw_c).
 \end{align*}
As $m,n\to \infty$, under similar regularization conditions (details in Appendix~\ref{app:MLE}) which guarantee the uniqueness of these estimates,  we can show that
\begin{align*}
    &\hat{W}^{t}_{c}(D_t, \hat{W}_{\phi}) - \hat{W}^{t}_{c}(\hat{W}_{\phi}) \rightarrow
  \mathcal{N}\big(0, \frac{J_c^t(\vw^{s*}_{\phi},\vw^{st*}_{c})^{-1}\mathcal{I}^t_c(\vw^{s*}_{\phi},\vw^{st*}_{c}) J_c^t(\vw^{s*}_{\phi},\vw^{st*}_{c}) ^{-1}}{m}\big),
\end{align*}
where
\begin{align}
    [\vw^{s*}_{\phi}, \vw^{s*}_{c}] &\triangleq \argmin_{[\vw_\phi,\vw_c]\in \mathcal{W}} D(P_{Z^s}\| f(\cdot|\vw_\phi,\vw_c)), \\
    \vw^{st*}_{c} &\triangleq \argmin_{\vw_c\in \mathcal{W}_c} D(P_{Z^t}\| f(\cdot|\vw^{s*}_\phi,\vw_c)),
\end{align}
and $J_c^t(\vw^{s*}_{\phi},\vw^{st*}_{c})$, $\mathcal{I}^t_c(\vw^{s*}_{\phi},\vw^{st*}_{c})$ stands for the expected Hessian matrix and Fisher information matrix over $\vw_c$ under target distribution, respectively. Detailed proofs are provided in Appendix~\ref{app:MLE_limit}.
As the Hessian matrix $H_c^*(D_t, W_{\phi}) = H_c^*(W_{\phi}) \to J_c^t(\vw^{s*}_{\phi},\vw^{st*}_{c}) $ as $m, n\to \infty$, by Proposition~\ref{Prop: large gamma beta}, we have
\begin{equation}\label{equ:gen_beta_mle}
    \overline{\text{gen}}_{\beta}(P_{D_t},P_{D_s})  = \mathcal{O}(\frac{d_c}{m}).
\end{equation}

\subsection{Excess Risk in MLE setting}\label{sec:excess_risk}

In this subsection, we further consider the excess risks of the $\alpha$-weighted-ERM algorithm and the two-stage-ERM algorithm in the aforementioned MLE setting when $m,n \to \infty$, and show that such analyses provide some intuitions in selecting different transfer learning algorithms. All the details are provided in Appendix \ref{app:excess_limit}.

The excess risk \citep{mohri2018foundations} is defined as the difference between the population risk achieved by the learning algorithm and that achieved by the optimal hypothesis given the knowledge of the true target distribution $P_{Z_t}$, i.e.,
\begin{align}\label{equ:excess_risk}
    \mathcal{E}_r(P_W)  \triangleq& \mathbb{E}_{P_{W,D_s,D_t}}[L_P(W,P_{D_t})] - L_P(\vw_t^*,P_{D_t}), \nn\\
    {\text{with}}\quad &\vw_t^{*} \triangleq \argmin_{\vw\in \mathcal{W}} L_P(\vw,P_{D_t}),
\end{align}where $\vw_t^{*}= \argmin_{\vw\in \mathcal{W}} D(P_{Z}^t \| f(\cdot|w))$ also holds in the MLE setting considered here.

\textbf{$\alpha$-weighted-ERM:} In general, a proper transfer learning algorithm should have small excess risk $\mathcal{E}_r$, which justifies the following approximation of the excess risk
\begin{align*}
    & \mathcal{E}_r(P_{\hat{W}_\alpha(D_s,D_t)}) \approx \frac{1}{2} \mathbb{E}_{P_{D_s,D_t}}\Big[\big\|\hat{W}_\alpha(D_s,D_t) - \vw_t^*\big\|^2_{J_t(\vw_t^*)}\Big] = \frac{1}{2}\|\vw^*_{\alpha}-\vw_t^* \|^2_{J_t(\vw_t^*)}+\frac{\mathrm{tr}\big(J_t(\vw_t^*) \Cov(\hat{W}_\alpha(D_s,D_t))\big)}{2}.
\end{align*}
As we can see from the above expression, the excess risk can be decomposed into squared bias and variance terms. The bias is caused by learning from the mixture of the source and target distributions instead of just the target distribution $P_{Z}^t$. In addition, it can be shown that $\mathrm{tr}(J_t(\vw_t^*) \Cov(\hat{W}_\alpha(D_s,D_t))) = \mathcal{O}(\frac{d}{m+n})$, which has the same order as the generalization error in \eqref{equ:gen_alpha_mle}.

\textbf{Two-stage-ERM:} In the two-stage algorithm, $\vw^{t*}$ can be written as $\vw^{t*} = [\vw^{t*}_{\phi}, \vw^{t*}_{c}]$, and using similar approximation, we have
\begin{align}
    & \mathcal{E}_r(P_{\hat{W}_{\phi}(D_s),\hat{W}^{t}_{c}(D_t, \hat{W}_{\phi})})
    \approx \frac{1}{2}\big\|[\vw^{s*}_{\phi},\vw^{st*}_{c}] - [\vw^{t*}_{\phi}, \vw^{t*}_{c}]\big\|^2_{J_t(\vw^{t*}_{\phi}, \vw^{t*}_{c})}
    +\frac{\mathrm{tr}\big(J_t(\vw^{t*}_{\phi}, \vw^{t*}_{c}) \Cov(\hat{W}_{\phi}(D_s),\hat{W}^{t}_{c}(D_t, \hat{W}_{\phi}))\big)}{2}.
\end{align}
Here the bias is caused by sharing the parameter $\vw^{s*}_{\phi}$ from the source distribution. If $\vw^{t*}_{\phi}=\vw^{s*}_{\phi}$, then $\vw^{st*}_{c} = \vw^{t*}_{c}$ and the bias is zero. It can be shown that the variance term scales as $O(\frac{d}{n}+\frac{d_c}{m})$. When $n \gg m$, it reduces to $\mathcal{O}(\frac{d_c}{m})$, which is the same as the generalization error in \eqref{equ:gen_beta_mle}.

In Table~\ref{tab:comparison}, we summarize the excess risk, and generalization error results for the two transfer learning algorithms studied in the paper and those of the standard supervised learning under MLE setting \citep{van2000asymptotic} as $m,n \to \infty$. The improvement of the excess risk for transfer learning algorithms comes from trading the variance induced by the lack of target samples with the bias introduced by the source distribution, which suggests that the choice of learning algorithm should depend on both source distribution and the number of samples $m,n$.

The bias term in the excess risks can be interpreted as another notion of discrepancy measure, which is algorithm-dependent, as $\vw^*_{\alpha}$ and $\vw^{s*}_{\phi}, \vw^{s*}_{c}$ are defined as the optimal parameters under different algorithms given the knowledge of both source and target distributions. Sometimes, these bias terms are more useful in choosing an algorithm than the discrepancy measure used in the literature. For example, consider the mean estimation example in Section~\ref{sec: example}, if we set $\vmu_s = \vmu_t$, $\sigma_s^2 \ll \sigma_t^2$, and let $m,n \to \infty$, then the bias term for both $\alpha$-weighted-ERM and two-stage-ERM should be zero, and transfer learning algorithms are preferred over the standard ERM. However, the KL divergence between the source and target distribution, which is proposed as a discrepancy measure in \citep{wu2020information}, would be  large.

The generalization error can be interpreted as the variance of the excess risk when $n \gg m$, and the analysis provided in the paper could help us to find a good balance in the bias and variance trade-off.

\section{Conclusion}
We provide an exact characterization of the generalization error for two Gibbs-based transfer learning algorithms, i.e., $\alpha$-weighted Gibbs algorithm and two-stage-ERM Gibbs algorithm, using conditional symmetrized KL information and divergence.  
Based on our results, we show that the benefits of transfer learning can be viewed as a bias-variance trade-off, and the bias term suggest a new notion of discrepancy measure,  which requires further investigation.


\clearpage

\bibliographystyle{ieeetr}
\bibliography{Refs}

\clearpage
\appendix

\section{Exact Characterization of Generalization Error Based on Symmetrized KL Information}\label{app: exact Characterization}

\subsection{$\alpha$-weighted Gibbs Algorithm}

\begin{reptheorem}{Theorem: Gibbs alpha Transfer result}\textbf{(restated)}
For the $\alpha$-weighted Gibbs algorithm, $0<\alpha<1$ and $\gamma>0$,
\begin{equation*}
    P_{W_\alpha|D_s,D_t}^\gamma (w_\alpha|d_s,d_t) = \frac{\pi(w_\alpha) e^{-\gamma L_E(w_\alpha,d_s,d_t)}}{V_\alpha(d_s,d_t,\gamma)},
\end{equation*}
the expected transfer generalization error is given by
\begin{equation*}
   \overline{\text{gen}}_{\alpha}(P_{D_s},P_{D_t}) =  \frac{I_{\mathrm{SKL}}(W_\alpha;D_t|D_s)}{\gamma\alpha}.
\end{equation*}
\end{reptheorem}
\begin{proof}
By the definition of conditional symmetrized KL information, we have
\begin{align}
    I_{\mathrm{SKL}}(W_\alpha;D_t|D_s)
    & = \mathbb{E}_{P_{D_s}}\Big[\mathbb{E}_{P_{W_\alpha,D_t|D_s}}\Big[\log\Big(\frac{P_{W_\alpha|D_s,D_t}^\gamma}{P_{W_\alpha|D_s}P_{D_t|D_s}}\Big)\Big]+\mathbb{E}_{P_{W_\alpha|D_s}P_{D_t|D_s}}\Big[\log\Big(\frac{P_{W_\alpha|D_s}P_{D_t|D_s}}{P_{W_\alpha|D_s,D_t}^\gamma}\Big)\Big]\Big] \nn\\
    &=\mathbb{E}_{P_{D_s}}\big[\mathbb{E}_{P_{W_\alpha,D_t|D_s}}[\log(P_{W_\alpha|D_s,D_t}^\gamma)]-\mathbb{E}_{P_{W_\alpha|D_s}P_{D_t|D_s}}[\log(P_{W_\alpha|D_s,D_t}^\gamma)]\big].
\end{align}
Combining with fact that $D_s$ and $D_t$ are independent, and plug in the posterior of $\alpha$-weighted Gibbs algorithm, we have
\begin{align}
    I_{\mathrm{SKL}}(W_\alpha;D_t|D_s)&=\mathbb{E}_{P_{D_s}}[\gamma\mathbb{E}_{P_{W_\alpha,D_t|D_s}}[L_E(W_\alpha,D_s,D_t)]-\gamma\mathbb{E}_{P_{W_\alpha|D_s}P_{D_t}}[L_E(W_\alpha,D_s,D_t)]]\nn\\
    &=\gamma\mathbb{E}_{P_{D_s}}[\mathbb{E}_{P_{W_\alpha,D_t|D_s}}[(1-\alpha) L_E(w_\alpha,d_s)+\alpha L_E(w_\alpha,d_t)]]\nn\\
    & \quad - \gamma\mathbb{E}_{P_{D_s}}[\mathbb{E}_{P_{W_\alpha|D_s}P_{D_t}}[(1-\alpha) L_E(w_\alpha,d_s)+\alpha L_E(w_\alpha,d_t)]]\nn\\
    &=\gamma \alpha\big[\mathbb{E}_{P_{W_\alpha,D_t,D_s}}[L_E(w_\alpha,d_t)] - [\mathbb{E}_{P_{W_\alpha,D_s}P_{D_t}}[L_E(w_\alpha,d_t)]\big]\\
    & = \gamma \alpha \overline{\text{gen}}_{\alpha}(P_{D_s},P_{D_t}). \qedhere \nn
\end{align}
\end{proof}

Due to the symmetry of the $\alpha$-weighted Gibbs algorithm, if we use $\overline{\text{gen}}_{\alpha}(P_{D_t},P_{D_s})$ to denote the generalization error of treating $P_{D_t}$ as source task and $D_s$ as the target, we can obtain that $\overline{\text{gen}}_{\alpha}(P_{D_t},P_{D_s}) =  \frac{I_{\mathrm{SKL}}(W_\alpha;D_s|D_t)}{\gamma\alpha}$.

It is also worthwhile to mention that the $\alpha$-weighted expected generalization error of both source and target tasks can be characterized in terms of symmetrized KL information as shown in the following Proposition.
\begin{proposition}\label{Prop: Combine MTL}
For $(\gamma,\pi(w_\alpha),L_E(w_\alpha,d_s,d_t))$-Gibbs algorithm and $0<\alpha<1$, we have
\begin{align}
    &\alpha  \overline{\text{gen}}_{\alpha}(P_{D_s},P_{D_t}) + (1-\alpha) \overline{\text{gen}}_{\alpha}(P_{D_t},P_{D_s})=\frac{I_{\mathrm{SKL}}(W_\alpha;D_t,D_s)}{\gamma}.
\end{align}
\end{proposition}
\begin{proof}
The symmetrized KL information can be written as
\begin{align}
    I_{\mathrm{SKL}}(W_\alpha;D_t,D_s) = \mathbb{E}_{P_{W_\alpha,D_t,D_s}}\big[\log(P_{W_\alpha|D_s,D_t}^\gamma)\big]-\mathbb{E}_{P_{W_\alpha}P_{D_t,D_s}}\big[\log(P_{W_\alpha|D_s,D_t}^\gamma)\big].
\end{align}
Plug in the posterior of $\alpha$-weighted Gibbs algorithm,
\begin{align}
    &{I_{\mathrm{SKL}}(W_\alpha;D_t,D_s)}\nn \\
    & =  \mathbb{E}_{P_{W_\alpha,D_t,D_s}}\big[-\gamma L_E(w_\alpha,d_s,d_t)\big]+\mathbb{E}_{P_{W_\alpha}P_{D_t,D_s}}\big[\gamma L_E(w_\alpha,d_s,d_t)\big] \nn \\
    & =-\gamma\mathbb{E}_{P_{W_\alpha,D_t,D_s}}[\alpha L_E(w_\alpha,d_t)+ (1-\alpha)L_E(w_\alpha,d_s)] + \gamma\mathbb{E}_{P_{W_\alpha}P_{D_t,D_s}}[\alpha L_E(w_\alpha,d_t)+(1-\alpha) L_E(w_\alpha,d_s)]\\
    & = \alpha \gamma \overline{\text{gen}}_{\alpha}(P_{D_s},P_{D_t}) + (1-\alpha) \gamma \overline{\text{gen}}_{\alpha}(P_{D_t},P_{D_s}).\nn \qedhere
\end{align}
\end{proof}
Note that the Proposition~\ref{Prop: Combine MTL} holds even for dependent source $D_s$ and target $D_t$ samples.

\subsection{Two-stage Gibbs Algrotihm}
\begin{reptheorem}{Theorem: Two-stage-ERM approach}\textbf{(restated)}
The expected transfer generalization error of the  two-stage Gibbs algorithm,
\begin{align*}
    P_{W_c^t|D_t,W_{\phi}}^\gamma (w_c^t|d_t,w_{\phi}) = \frac{\pi(w_c^t) e^{-\gamma L_E^{S_2}(w_\phi,w_c^t,d_t)}}{V_\beta(w_\phi,d_t,\gamma)},
\end{align*}
is given by
\begin{align*}
   &\overline{\text{gen}}_{\beta}(P_{D_s},P_{D_t})=\frac{ I_{\mathrm{SKL}}(D_t;W_c^t|W_\phi)}{\gamma}.
\end{align*}
\end{reptheorem}
\begin{proof}
In the second stage we freeze the share parameters $W_{\phi}$, and we will update the specific target task parameter. Thus,
\begin{align}
    &I_{\mathrm{SKL}}(W_c^t;D_t|W_{\phi}) \nn\\
    & = \mathbb{E}_{P_{W_\phi}}\big[\mathbb{E}_{P_{W_c^t,D_t|W_{\phi}}}[\log(P_{W_c^t|D_t,W_{\phi}})] - \mathbb{E}_{P_{W_c^t|W_{\phi}}P_{D_t|W_{\phi}}}[\log(P_{W_c^t|D_s,W_{\phi}})]\big] \nn\\
    &=\gamma\left(\mathbb{E}_{P_{W_\phi}}\big[\mathbb{E}_{P_{W_c^t|W_{\phi}}P_{D_t|W_{\phi}}}[L_E^{S2}(W_{\phi},W_c^t,D_t)]- \mathbb{E}_{P_{W_c^s,D_t|W_{\phi}}}[L_E^{S2}(W_{\phi},W_c^t,D_t)]\big]\right) \\
    & =\gamma \overline{\text{gen}}_{\beta}(P_{D_s},P_{D_t}).\nn \qedhere
\end{align}
\end{proof}

\section{Example: Mean Estimation}\label{app: Mean Estimation}

\subsection{Symmetrized KL Divergence}
The following lemma from \citep{palomar2008lautum} characterizes the mutual and lautum information for the Gaussian channel.

\begin{lemma}{\citep[Theorem 14]{palomar2008lautum}}\label{lemma:Gaussian}
Consider the following model
\begin{equation}
    \mY = \mA \mX+\mN_{\mathrm{G}},
\end{equation}
where $\mX \in \mathbb{R}^{d_X}$ denotes the input
random vector with zero mean (not necessarily
Gaussian), $\mA \in \mathbb{R}^{d_Y \times d_X}$ denotes the linear transformation undergone by the input, $\mY\in \mathbb{R}^{d_Y}$ is the
output vector, and $\mN_{\mathrm{G}}\in \mathbb{R}^{d_Y}$ is a
Gaussian noise vector independent of $\mX$. The input and the
noise covariance matrices are given by
$\mSigma$ and $\mSigma_{N_{\mathrm{G}}}$.
Then, we have
\begin{align}
    I(\mX;\mY) &= \frac{1}{2}\mathrm{tr}\big(\mSigma_{N_{\mathrm{G}}}^{-1} \mA \mSigma \mA^\top \big) - D\big(P_\mY\|P_{N_{\mathrm{G}}} \big),  \\
    L(\mX;\mY) &= \frac{1}{2}\mathrm{tr}\big(\mSigma_{N_{\mathrm{G}}}^{-1} \mA \mSigma \mA^\top \big) + D\big(P_\mY\|P_{N_{\mathrm{G}}}).
\end{align}
\end{lemma}
In the $\alpha$-weighted Gibbs algorithm, the output $W_\alpha$ can be written as
\begin{align}\label{equ:Gaussian_alpha}
    W_\alpha & = \frac{\sigma_1^2 }{\sigma_0^2}\vmu_0 +\frac{\sigma_1^2 }{\sigma^2}\big(\sum_{i=1}^{ n} Z^s_i +\sum_{j=1}^{m} Z^t_j \big) +N  = \frac{\sigma_1^2 }{\sigma^2}\sum_{j=1}^m (Z_j^t-\vmu_t) + \frac{\sigma_1^2 }{\sigma_0^2}\vmu_0+\frac{m \sigma_1^2  }{\sigma^2}\vmu_t + \frac{\sigma_1^2 }{\sigma^2}\sum_{i=1}^n Z_i^s + N,
\end{align}
where $N\sim \mathcal{N}(0,\sigma^2_1 I_d)$, and $\sigma_1^2=\frac{\sigma_0^2 \sigma^2}{(m+n)\sigma_0^2 +\sigma^2}$. For fixed sources training sample $d_s$, we can set $P_{N_{\mathrm{G}}} \sim \mathcal{N}(\frac{\sigma_1^2 }{\sigma_0^2}\vmu_0+\frac{m \sigma_1^2  }{\sigma^2}\vmu_t + \frac{\sigma_1^2 }{\sigma^2}\sum_{i=1}^n z_i^s, \sigma^2_1 I_d)$ and $\mSigma = \sigma_t^2 I_{nd}$ in Lemma~\ref{lemma:Gaussian} gives
\begin{align}
    I_{\mathrm{SKL}}(W_\alpha;D_t|D_s=d_s) = \mathrm{tr}\big(\mSigma_{N_{\mathrm{G}}}^{-1} \mA \mSigma \mA^\top \big) = \mathrm{tr}\big(\frac{\sigma_t^2}{\sigma_1^2} \mA  \mA^\top \big).
\end{align}
Noticing that $\mA  \mA^\top = \frac{m \sigma_1^4}{ \sigma^4}I_{d}$ and taking expectation over $P_S$, we have
\begin{align}
    I_{\mathrm{SKL}}(W_\alpha;D_t|D_s) &= \frac{md \sigma_0^2 \sigma_t^2}{((m+n)\sigma_0^2 + \sigma^2)\sigma^2}.
\end{align}

For the two-stage Gibbs algorithm, the output $W_c^t$ can be written as
\begin{align}\label{equ:Gaussian_two}
    W_c^t & = \frac{\sigma_c^2 }{\sigma_0^2}\vmu_{0,c} +\frac{\sigma_c^2 }{\sigma^2} \sum_{j=1}^{m} Z^t_{j,c} + N_c  = \frac{\sigma_c^2 }{\sigma^2}\sum_{j=1}^m (Z_{j,c}^t-\vmu_{t,c}) + \frac{\sigma_c^2 }{\sigma_0^2}\vmu_{0,c}+\frac{n \sigma_c^2  }{\sigma^2}\vmu_{t,c} + N_c,
\end{align}
where $N_c\sim \mathcal{N}(0,\sigma^2_c I_{d_c})$, $\sigma_c^2=\frac{\sigma_0^2 \sigma^2}{m\sigma_0^2 +\sigma^2}$, and subscript $c$ stands for the task-specific component of the parameters. Since $W_c^t$ is independent of the source samples, setting $P_{N_{\mathrm{G}}} \sim \mathcal{N}(\frac{\sigma_c^2 }{\sigma_0^2}\vmu_{0,c}+\frac{n \sigma_c^2  }{\sigma^2}\vmu_{t,c}, \sigma^2_c I_{d_c})$ and $\mSigma = \sigma_t^2 I_{nd_c}$ in Lemma~\ref{lemma:Gaussian} gives
\begin{align}
   I_{\mathrm{SKL}}(W_c^t;D_t|W_{\phi}) =  \mathrm{tr}\big(\mSigma_{N_{\mathrm{G}}}^{-1} \mA \mSigma \mA^\top \big) = \mathrm{tr}\big(\frac{\sigma_t^2}{\sigma_c^2} \mA  \mA^\top \big)=\frac{md_c \sigma_0^2 \sigma_t^2}{(m\sigma_0^2 + \sigma^2)\sigma^2},
\end{align}
where the last step follows due to the fact $\mA  \mA^\top = \frac{m \sigma_c^4}{ \sigma^4}I_{d_c}$ in this case.

\subsection{Effect of Source samples}
As shown in \eqref{equ:alpha_Gaussian} and \eqref{equ:beta_Gaussian}, the transfer generalization errors of this mean estimation problem only depend on the number of samples of $D_s$, and do not depend on the distribution $P_{D_s}$. In this subsection, we will show that, though different sources samples (distribution) do not change generalization error, they will influence the population risks and excess risks.

In this mean estimation example, the population risk of any $W$ can be decomposed into \begin{align}
    L_P(W,P_{D_t}) &= \mathbb{E}_{Z_t}[\|W-Z_t\|_2^2] = \mathbb{E}_{Z_t}[\|W-\mathbb{E}[W]+\mathbb{E}[W]-\vmu_t + \vmu_t - Z_t\|_2^2] \nn\\
    &= \|\mathbb{E}[W]-\vmu_t\|_2^2 + \mathrm{tr}(\Cov[W]) +  d\sigma_t^2,
\end{align}
where the first term, $\|\mathbb{E}[W]-\vmu_t\|_2^2$, is the squared bias, and the second term, $\mathrm{tr}(\Cov[W])$, is the variance. It is easy to verify that the optimal $\vw^* = \argmin L_P(W,P_{D_t})$ is just the target mean $\vmu_t$, and $L_P(\vw^*,P_{D_t}) = d\sigma_t^2$, then the excess risk defined in \eqref{equ:excess_risk}  can be written as,
\begin{align}
    \mathcal{E}_r(P_W)  = \|\mathbb{E}[W]-\vmu_t\|_2^2 + \mathrm{tr}(\Cov[W]).
\end{align}
For the $\alpha$-weighted Gibbs algorithm in \eqref{equ:Gaussian_two}, it can be shown that
\begin{align}
   \mathrm{Bias} =  \mathbb{E}[W_\alpha]-\vmu_t &= \frac{\sigma^2(\vmu_0-\vmu_t)+n\sigma_0^2(\vmu_s-\vmu_t)}{(m+n)\sigma_0^2+\sigma^2},\\
   \mathrm{tr}(\Cov[W_\alpha]) &= \frac{d\sigma_1^4}{\sigma^4}(n\sigma_s^2+m\sigma_t^2) +d \sigma_1^2.
\end{align}
The Bias term will be zero if $\vmu_0 = \vmu_s = \vmu_t$. Thus, the excess risk of $\alpha$-weighted Gibbs algorithm will be minimized when $\vmu_s = \vmu_t$ and $\sigma_s^2=0$, which is equivalent to the case that the target mean $\vmu_t$ is known.

For the two-stage Gibbs algorithm, if we learn the first $d_\phi$ components $\vmu_\phi \in \mathbb{R}^{d_\phi}$ using the Gibbs algorithm with $(\frac{n}{2\sigma^2}, \mathcal{N}(\vmu_{1,\phi},\sigma^2_0 I_{d_\phi}), L_E^{S1}(\vw_{\phi},\vw_c^s,d_s))$, and use the $(\frac{m}{2\sigma^2}, \mathcal{N}(\vmu_{2,c},\sigma^2_0 I_{d_c}), L_E^{S2}(\vmu_{\phi},\vw_c^t,d_t))$-Gibbs algorithm to learn the remain $d_c$ components in the second stage, it can be shown that
\begin{align}
   \mathrm{Bias}_\phi =  \mathbb{E}[W_\phi]-\vmu_{t,\phi} &= \frac{\sigma^2(\vmu_{1,\phi}-\vmu_{t,\phi})+n\sigma_0^2(\vmu_{s,\phi}-\vmu_{t,\phi})}{n\sigma_0^2+\sigma^2},\\
   \mathrm{Bias}_c =  \mathbb{E}[W_c^t]-\vmu_{t,c} &= \frac{\sigma^2(\vmu_{2,c}-\vmu_{t,c})}{m\sigma_0^2+\sigma^2},\\
   \mathrm{tr}(\Cov[W_\phi]) &= \frac{nd_\phi\sigma_\phi^4\sigma_s^2}{\sigma^4} +d_\phi \sigma_\phi^2,\\
   \mathrm{tr}(\Cov[W_c^t]) &= \frac{md_c\sigma_c^4\sigma_t^2}{\sigma^4} +d_c \sigma_c^2,
\end{align}
with $\sigma_\phi^2=\frac{\sigma_0^2 \sigma^2}{n\sigma_0^2 +\sigma^2}$ and $\sigma_c^2=\frac{\sigma_0^2 \sigma^2}{m\sigma_0^2 +\sigma^2}$. The excess risk of the two-stage  Gibbs algorithm will be minimized when $\vmu_{s,\phi} = \vmu_{t,\phi}$ and $\sigma_s^2=0$, i.e., the optimal shared parameter $\vmu_{t,\phi}$ is known.

\section{Expected Transfer Generalization Error Upper Bound for General Learning Algorithm}\label{app: proof Theorem: General result}

\subsection{Preliminaries}
We first provide some preliminaries for our proofs in this section by introducing the notion of cumulant generating function, which characterizes different tail behaviors of random variables.
\begin{definition}
The cumulant generating function (CGF) of a random variable $X$ is defined as
\begin{equation}
\Lambda_X(\lambda) \triangleq \log \mathbb{E}[e^{\lambda(X-\mathbb{E}X)}].
\end{equation}
\end{definition}
Assuming $\Lambda_X(\lambda)$ exists, it can be verified that $\Lambda_X(0)=\Lambda_X'(0)=0$, and that it is convex.

\begin{definition}
For a convex function $\psi$ defined on the interval $[0,b)$, where $0<b\le \infty$, its Legendre dual $\psi^\star$ is defined as
\begin{equation}
  \psi^\star(x) \triangleq \sup_{\lambda \in [0,b)} \big(\lambda x-\psi(\lambda)\big).
\end{equation}
\end{definition}

The following lemma characterizes a useful property of the Legendre dual and its inverse function.
\begin{lemma}\label{lemma:psi_star}\citep[Lemma 2.4]{boucheron2013concentration}
Assume that $\psi(0)= \psi'(0) = 0$. Then $\psi^\star(x)$ defined above is a non-negative convex and non-decreasing function on $[0,\infty)$ with $\psi^\star(0) = 0$. Moreover, its inverse function $\psi^{\star -1}(y) = \inf\{x\ge 0: \psi^\star(x)\ge y\}$ is concave, and can be written as
\begin{equation}
  \psi^{\star -1}(y) = \inf_{\lambda \in [0,b)} \Big( \frac{y+\psi(\lambda)}{\lambda} \Big),\quad b> 0.
\end{equation}
\end{lemma}
We consider the distributions with the following tail behaviors in the appendices:
\begin{itemize}
\item \textbf{Sub-Gaussian:} A random variable $X$ is  $\sigma$-sub-Gaussian, if $\psi(\lambda)= \frac{\sigma^2\lambda^2}{2}$ is an upper bound on $\Lambda_X(\lambda)$, for $\lambda \in \mathbb{R}$. Then by Lemma~\ref{lemma:psi_star}, $$\psi^{\star -1}(y)=\sqrt{2\sigma^2y}.$$
\item \textbf{Sub-Exponential:} A random variable $X$ is  $(\sigma_e^2,b)$-sub-Exponential, if $\psi(\lambda)= \frac{\sigma_e^2\lambda^2}{2}$ is an upper bound on $\Lambda_X(\lambda)$, for $0 \leq |\lambda|\leq \frac{1}{b} $ and $b > 0$. Using Lemma~\ref{lemma:psi_star}, we have
$$
\psi^{\star -1}(y)=
    \begin{cases}                       &\sqrt{2\sigma_e^2y}, \quad \textit{if }y\leq \frac{\sigma_e^2}{2b};\\
    &by+\frac{\sigma_e^2}{2b},\quad \textit{otherwise.}
    \end{cases}
$$
\item \textbf{Sub-Gamma:} A random variable $X$ is  $\Gamma(\sigma_s^2,c_s)$-sub-Gamma \citep{zhang2020concentration}, if $\psi(\lambda)=  \frac{\lambda^2 \sigma_s^2}{2(1-c_s |\lambda|)}$ is an upper bound on $\Lambda_X(\lambda)$, for $0 < |\lambda| < \frac{1}{c_s} $ and $c_s>0$. Using Lemma~\ref{lemma:psi_star}, we have $$\psi^{\star -1}(y)=\sqrt{2\sigma_s^2y}+c_s y.$$
\end{itemize}

\subsection{Proof of Theorem~\ref{Theorem: General result}}
We prove a more general form of Theorem~\ref{Theorem: General result} as follows:
\begin{theorem}\label{Theorem: general with psi}
Suppose that the target training samples $D_t=\{Z_j^t\}_{j=1}^m$ are i.i.d generated from the distribution $P_Z^t$ and the loss function $\ell(w,Z)$ satisfies $ \Lambda_{\ell(w,Z)}(\lambda)\leq \psi(-\lambda)$, for $\lambda \in (-b,0)$ and $ \Lambda_{\ell(w,Z)}(\lambda)\leq \psi(\lambda)$, for $\lambda \in (0,b)$ and $b>0$ under the distribution $P_Z^t \otimes P_W$. The following upper bound holds:
\begin{align}
    &|\overline{\text{gen}}(P_{W|D_s,D_t},P_{D_s},P_{D_t})|
\leq
   \psi^{\star -1}(\frac{ I(W;D_t|D_s)}{m}).
\end{align}
\end{theorem}
\begin{proof}
The generalization error can be written as
\begin{align}\label{Eq: decomp gen psi}
   |\overline{\text{gen}}(P_{W|D_s,D_t},P_{D_s},P_{D_t})|\leq \frac{1}{m}\sum_{i=1}^m |\mathbb{E}_{P_{W,Z_i^t}}[ \ell(W,Z_i^t)]-\mathbb{E}_{P_{W}\otimes P_{Z}^t}[\ell(W,Z^t)]|.
\end{align}
Using the Donsker–Varadhan variational representation  \citep{boucheron2013concentration}, for all $\lambda \in (-b,+b)$,
\begin{align}
    D(P_{W,Z_i^t|d_s}\| P_{W|d_s}\otimes P_{Z}^t)\geq \mathbb{E}_{P_{W,Z_i^t|d_s}}[\lambda \ell(W,Z_i^t)]-\log(\mathbb{E}_{P_{W|d_s}\otimes P_{Z}^t}[e^{\lambda \ell(W,Z^t)}]).
\end{align}
Taking expectation respect to $D_s$ over both sides, then we have
\begin{align}\label{Eq: jensen 1}
    I(W;Z_i^t|D_s)&\geq \mathbb{E}_{P_{W,Z_i^t}}[\lambda \ell(W,Z_i^t)]-\mathbb{E}_{P_{D_s}}[\log(\mathbb{E}_{P_{W|D_s}\otimes P_{Z}^t}[e^{\lambda \ell(W,Z^t)}])] \nn\\
    &
    \geq \mathbb{E}_{P_{W,Z_i^t}}[\lambda \ell(W,Z_i^t)]-\log(\mathbb{E}_{P_{W}\otimes P_{Z}^t}[e^{\lambda \ell(W,Z^t)}])\nn \\
    &
    \geq \lambda(\mathbb{E}_{P_{W,Z_i^t}}[ \ell(W,Z_i^t)]-\mathbb{E}_{P_{W}\otimes P_{Z}^t}[\ell(W,Z^t)])-\psi(\lambda).
\end{align}
Using similar approach as in \citep[Theorem~1]{bu2020tightening},
\begin{align}\label{Eq: psi bound}
    |\mathbb{E}_{P_{W,Z_i^t}}[ \ell(W,Z_i^t)]-\mathbb{E}_{P_{W}\otimes P_{Z}^t}[\ell(W,Z^t)]|\leq \psi^{\star -1}(I(W;Z_i^t|D_s)).
\end{align}
Now by combining \eqref{Eq: decomp gen psi} and \eqref{Eq: psi bound}, we have:
\begin{align}\label{Eq: concave}
   |\overline{\text{gen}}(P_{W|D_s,D_t},P_{D_s},P_{D_t})|&\leq \frac{1}{m}\sum_{i=1}^m \psi^{\star -1}(I(W;Z_i^t|D_s)) \nn\\
   &\le  \psi^{\star -1}\Big(\frac{1}{m}\sum_{i=1}^m I(W;Z_i^t|D_s)\Big) \nn \\
   &\leq \psi^{\star -1}\Big(\frac{I(W,D_t|D_s)}{m}\Big),
\end{align}
where the inequality follows due to the concavity of $\psi^{\star -1}$ function and the Independence between $Z_i^t$.
\end{proof}
\begin{reptheorem}{Theorem: General result}
\textbf{(restated)}
Suppose that the target training samples $D_t=\{Z_j^t\}_{j=1}^m$ are i.i.d generated from the distribution $P_Z^t$, and the non-negative loss function $\ell(w,Z)$ is $\sigma$-sub-Gaussian
under the distribution $P_Z^t \otimes P_W$. Then, the following upper bound holds
\begin{align*}
   &|\overline{\text{gen}}(P_{W|D_s,D_t},P_{D_s},P_{D_t})|
\leq
   \sqrt{\frac{2 \sigma^2}{m} I(W;D_t|D_s)}.
\end{align*}
\end{reptheorem}
\begin{proof}
For $\sigma$-subgaussian assumption, we have $\psi^{\star -1}(y)=\sqrt{2\sigma^2y}$ in Theorem~\ref{Theorem: general with psi} and this completes the proof.
\end{proof}

\begin{remark}
Similar upper bound on the expected transfer generalization error in Theorem~\ref{Theorem: General result} holds by considering a different assumption that the loss function $\ell(w,Z)$ is $\sigma$-sub-Gaussian
under the distribution $P_Z^t$ for all $w \in \mathcal{W}$.
\end{remark}


\subsection{Other Tail Distributions}

Using Theorem~\ref{Theorem: general with psi}, we can also provide upper bounds on the expected transfer generalization error for any general learning algorithms under sub-Exponential and sub-Gamma assumptions.

\begin{corollary}[Sub-Exponential]\label{Cor: sub exponential general}
Suppose that the target training samples $D_t=\{Z_j^t\}_{j=1}^m$ are i.i.d generated from the distribution $P_Z^t$, and the non-negative loss function $\ell(w,Z)$ $(\sigma_e^2,b)$-sub-Exponential under distribution $P_Z^t \otimes P_W$. Then the following upper bound holds
\begin{align}
   &|\overline{\text{gen}}(P_{W|D_s,D_t},P_{D_s},P_{D_t})|
\leq
   \begin{cases}
   &\sqrt{2\sigma_e^2\frac{I(W;D_t|D_s)}{m}}, \quad \textit{if }\frac{I(W;D_t|D_s)}{m} \leq \frac{\sigma_e^2}{2b};\\
    &b\frac{I(W;D_t|D_s)}{m}+\frac{\sigma_e^2}{2b},\quad \textit{otherwise.}
\end{cases}.
\end{align}
\end{corollary}

\begin{corollary}[Sub-Gamma]\label{Cor: sub gamma general}
Suppose that the target training samples $D_t=\{Z_j^t\}_{j=1}^m$ are i.i.d generated from the distribution $P_Z^t$, and the non-negative loss function $\ell(w,Z)$ is $\Gamma(\sigma_s^2,c_s)$-sub-Gamma under distribution $P_Z^t \otimes P_W$. Then, the following upper bound holds
\begin{align}
   &|\overline{\text{gen}}(P_{W|D_s,D_t},P_{D_s},P_{D_t})|
\leq
  \sqrt{2\sigma_s^2\frac{I(W;D_t|D_s)}{m} }+c_s \frac{I(W;D_t|D_s)}{m}.
\end{align}
\end{corollary}

\section{Distribution-free Upper Bound on Generalization Error}\label{app: Distribution-free Upper Bound}



\begin{reptheorem}{Theorem:  distribution-free upper bound weighted}\textbf{(restated)}
Suppose that the target training samples $D_t=\{Z_j^t\}_{j=1}^m$ are i.i.d generated from the distribution $P_Z^t$, and the non-negative loss function $\ell(w,z)$ is $\sigma_\alpha$-sub-Gaussian
under the distribution $P_Z^t \otimes P_{W_\alpha}$.
If we further assume $C_\alpha\le \frac{L(W_\alpha;D_t|D_s)}{I(W_\alpha;D_t|D_s)}$ for some $C_\alpha \ge 0$, then for the $\alpha$-weighted Gibbs algorithm and $0<\alpha<1$,
\begin{align*}
    &\overline{\text{gen}}_{\alpha}(P_{D_s},P_{D_t}) \leq \frac{2\sigma_\alpha^2\gamma\alpha}{(1+C_\alpha)m}.
\end{align*}
\end{reptheorem}
\begin{proof}
By equation \eqref{eq: General upper bound transfer} in Theorem~\ref{Theorem: General result}, we have
\begin{align}\label{Eq: Transfer const 1}
    \overline{\text{gen}}_\alpha(P_{D_s},P_{D_t})
    &=\frac{I_{\mathrm{SKL}}(W_\alpha;D_t|D_s)}{\gamma\alpha} \\\nonumber
    &\le \sqrt{\frac{2\sigma^2I(W_\alpha;D_t|D_s)}{m}}.
\end{align}
As we have $I(W_\alpha;D_t|D_s)(1+C_\alpha)\leq I_{\mathrm{SKL}}(W_\alpha;D_t|D_s)$ in the assumption, the following upper bound holds:
\begin{align}
    \frac{I(W_\alpha;D_t|D_s)(1+C_\alpha)}{\gamma\alpha} \le \sqrt{\frac{2\sigma_\alpha^2I(W_\alpha;D_t|D_s)}{m}},
\end{align}
which implies that
\begin{equation}\label{Eq: MI upper bound transfer}
    I(W_\alpha;D_t|D_s)\leq \frac{2\sigma_\alpha^2\gamma^2\alpha^2}{(1+C_\alpha)^2m}.
\end{equation}
Combining \eqref{Eq: MI upper bound transfer} with \eqref{Eq: Transfer const 1} completes the proof.
\end{proof}

\begin{reptheorem}{Theorem: distribution-free upper bound two-stage}\textbf{(restated)}
Suppose that the target training samples $D_t=\{Z_j^t\}_{j=1}^m$ are i.i.d generated from the distribution $P_Z^t$, and the non-negative loss function $\ell(w,z)$ is $\sigma_\beta$-sub-Gaussian
under distribution $P_Z^t\otimes P_{W_c^t|W_\phi=w_\phi}$ for all $w_\phi \in \mathcal{W}_\phi$.
If we further assume $C_\beta\le \frac{L(W_c^t;D_t|W_\phi)}{I(W_c^t;D_t|W_\phi)}$ for some $C_\beta \ge 0$, then for the two-stage Gibbs algorithm,
\begin{align*}
    P_{W_c^t|D_t,W_{\phi}}^\gamma (w_c^t|d_t,w_{\phi}) = \frac{\pi(w_c^t) e^{-\gamma L_E^{S_2}(w_\phi,w_c^t,d_t)}}{V_\beta(w_\phi,d_t,\gamma)},
\end{align*}
 we have
\begin{align*}
   \overline{\text{gen}}_{\beta}(P_{D_s},P_{D_t})\leq \frac{2\sigma_\beta^2\gamma}{(1+C_\beta)m}.
\end{align*}
\end{reptheorem}

\begin{proof}
Using Theorem~\ref{Theorem: General result} by considering $W=(W_C^t,W_\phi)$,
\begin{align*}
   &|\overline{\text{gen}}_\beta(P_{D_s},P_{D_t})|
\leq
   \sqrt{\frac{2 \sigma^2}{m} I(W_c^t,W_\phi;D_t|D_s)}.
\end{align*}
Now, based on chain rule for mutual information we have
\begin{align*}
    I(W_c^t,W_\phi;D_t|D_s)&=I(W_\phi;D_t|D_s)+I(W_c^t;D_t|D_s,W_\phi)\\&=I(W_c^t;D_t|W_\phi),
\end{align*}
where $I( W_\phi;D_t|D_s) = 0$ due to the fact that $W_\phi$ is independent from $D_t$ given $D_s$, and $I( W_c^t;D_t|W_\phi, D_s) = I( W_c^t;D_t|W_\phi)$ since $D_s \perp (W_c^t,D_t) | W_\phi$.

Using Theorem~\ref{Theorem: Two-stage-ERM approach}, it can be shown that
\begin{align}\label{Eq: Transfer const 2}
   &\overline{\text{gen}}_{\beta}(P_{D_s},P_{D_t})=\frac{ I_{\mathrm{SKL}}(D_t;W_c^t|W_\phi)}{\gamma}\leq \sqrt{\frac{2\sigma_\beta^2}{m}I(W_c^t;D_t|W_\phi)}.
\end{align}
As we have $I(W_c^t;D_t|W_\phi)(1+C_\beta)\leq I_{\mathrm{SKL}}(W_c^t;D_t|W_\phi)$, the following bound holds:
\begin{align}
    \frac{I(W_c^t;D_t|W_\phi)(1+C_\beta)}{\gamma} \le \sqrt{\frac{2\sigma_\beta I(W_c^t;D_t|W_\phi)}{m}},
\end{align}
which implies that
\begin{equation}\label{Eq: MI upper bound transfer2}
    I(W_c^t;D_t|W_\phi)\leq \frac{2\sigma_\beta^2\gamma^2}{(1+C_\beta)^2m}.
\end{equation}
Combining \eqref{Eq: MI upper bound transfer2} with \eqref{Eq: Transfer const 2} completes the proof.
\end{proof}

We could provide distribution-free upper bounds under sub-Exponential and sub-Gamma assumption using  similar approach as in Theorem~\ref{Theorem:  distribution-free upper bound weighted} and Theorem~\ref{Theorem:  distribution-free upper bound two-stage} for $\alpha$-weighted Gibbs algorithm and two-stage Gibbs algorithm, respectively.

\textbf{sub-Exponential:}  For $\alpha$-weighted Gibbs algorithm, we assume that the loss function is $(\sigma_{\alpha,e}^2,b_\alpha)$-sub-Exponential under distribution $P_Z^t \otimes P_{W_\alpha}$. And for two-stage Gibbs algorithm, we assume that the loss function is $(\sigma_{\beta,e}^2,b_\beta)$-sub-Exponential under distribution $P_Z^t\otimes P_{W_c^t|W_\phi=w_\phi}$ for all $w_\phi \in \mathcal{W}_\phi$. We provide the results in Table~\ref{Table:tail Bound Comparison}. Denote $B_\alpha\triangleq \ceil{\frac{\gamma\alpha b_\alpha}{1+C_\alpha}}$, $B_\beta\triangleq\ceil{\frac{\gamma b_\beta}{1+C_\beta}}$, $I_\alpha\triangleq \frac{2b_\alpha I(W_\alpha;D_t|D_s)}{\sigma_{\alpha,e}^2}$ and $I_\beta\triangleq \frac{2b_\beta I(W_c^t;D_t|W_\phi)}{\sigma_{\beta,e}^2}$ in Table~\ref{Table:tail Bound Comparison}.

\textbf{sub-Gamma:} For $\alpha$-weighted Gibbs algorithm, we assume that the loss function is $\Gamma(\sigma_{\alpha,s}^2,c_{\alpha,s})$-sub-Gamma under distribution $P_Z^t \otimes P_{W_\alpha}$ and $m > \frac{\gamma\alpha c_{\alpha,s}}{(1+C_\alpha)}$. For two-stage Gibbs algorithm, we assume that the loss function is $\Gamma(\sigma_{\beta,s}^2,c_{\beta,s})$-sub-Gamma under distribution $P_Z^t\otimes P_{W_c^t|W_\phi=w_\phi}$ for all $w_\phi \in \mathcal{W}_\phi$ and $m > \frac{\gamma c_{\beta,s}}{(1+C_\beta)}$. We provide the results in Table~\ref{Table:tail Bound Comparison}.

\begin{table*}[ht!]
  \caption{Distribution-free Upper Bounds under different Tail Distributions.}
  \label{Table:tail Bound Comparison}
  \centering
	\begin{tabular}{ccc}
	\\
    	\toprule
    	 & \textbf{sub-Exponential} & \textbf{sub-Gamma} \\
    	\midrule

    	\makecell{$\alpha$-weighted\\ Gibbs Algorithm }
    	& $
    	\begin{cases}
         \frac{2\sigma_{\alpha,e}^2 \gamma\alpha}{m (1+C_\alpha)}, &\textit{if  } m\geq I_\alpha; \\
          \frac{\sigma_{\alpha,e}^2}{2b_\alpha}\Big(\frac{\gamma\alpha b_\alpha }{(m(1+C_\alpha)-\gamma\alpha b_\alpha) }+1\Big), & \textit{if  } B_\alpha <m < I_\alpha
       \end{cases}$
    	& $\frac{2\sigma_{\alpha,s}^2 \gamma\alpha}{(1+C_\alpha)m-\gamma\alpha c_{\alpha,s}}
    \Big(1+\frac{\gamma\alpha c_{\alpha,s}}{(1+C_\alpha)m-\gamma\alpha c_{\alpha,s}}\Big)$ \\
    	\midrule

    	\makecell{Two-stage \\ Gibbs Algorithm}
    	&$
        \begin{cases}
              \frac{2\sigma_{\beta,e}^2 \gamma}{m (1+C_\beta)}, &\textit{if  } m\geq I_\beta; \\
                \frac{\sigma_{\beta,e}^2}{2b_\beta}\Big(\frac{\gamma b_\beta }{(m(1+C_E)-\gamma b_\beta) }+1\Big), & \textit{if  } B_\beta <m < I_\beta
       \end{cases}$
    	& $\frac{2\sigma_{\beta,s}^2 \gamma}{(1+C_\beta)m-\gamma c_{\beta,s}}
    \Big(1+\frac{\gamma c_{\beta,s}}{(1+C_\beta)m-\gamma c_{\beta,s}}\Big)$  \\

    	\bottomrule
  \end{tabular}
\end{table*}
\normalsize

\section{Exact Characterization of Generalization Error Based on Symmetrized KL divergence}\label{App: new rep skl divergence}

We first present the following Lemma to prove the results related to symmetrized KL divergence.
\begin{lemma}\label{Lemma: exact SKL two Gibbs}
 Denote the $(\gamma,\pi(w),L_E(w,d_t))$-Gibbs algorithm as $P_{W|D_t}^{\gamma}$ and the $(\gamma,\pi(w),L_P(w,P_{D_t}))$-Gibbs algorithm as $P_{W}^{\gamma,L_{P_{D_t}}}$. Then, the following equality holds for these two Gibbs distributions with the same inverse temperature and prior distribution
\begin{align}
\mathbb{E}_{\Delta(P_{W|D_t=d_t}^{\gamma},P_{W}^{\gamma,L_{P_{D_t}}})}[L_P(W,P_{D_t})-L_E(W,d_t)]=\frac{D_{\mathrm{SKL}}(P_{W|D_t=d_t}^{\gamma}\|P_{W}^{\gamma,L_{P_{D_t}}})}{\gamma},
\end{align}
where $\mathbb{E}_{\Delta(P_{W|D_t=d_t}^{\gamma},P_{W}^{\gamma,L_{P_{D_t}}})}[f(W)]=\mathbb{E}_{P_{W|D_t=d_t}^{\gamma}}[f(W)]-\mathbb{E}_{P_{W}^{\gamma,L_{P_{D_t}}}}[f(W)]$.
\end{lemma}

\begin{proof}
\begin{align}
    D_{\mathrm{SKL}}(P_{W|D_t=d_t}^{\gamma}\|P_{W}^{\gamma,L_{P_{D_t}}})&=\int_\mathcal{W} (P_{W|D_t=d_t}^{\gamma}-P_{W}^{\gamma,L_{P_{D_t}}}) \log\left(\frac{P_{W|D_t=d_t}^{\gamma}}{P_{W}^{\gamma,L_{P_{D_t}}}}\right)dw \nn \\
    &= \int_\mathcal{W} (P_{W|D_t=d_t}^{\gamma}-P_{W}^{\gamma,L_{P_{D_t}}}) \log(e^{-\gamma (L_E(w,d_t)-L_P(w,P_{D_t}))})dw\\
    &=\gamma \mathbb{E}_{\Delta(P_{W|D_t=d_t}^{\gamma},P_{W}^{\gamma,L_{P_{D_t}}})}[L_P(W,P_{D_t})-L_E(W,d_t)]. \nn \qedhere
\end{align}
\end{proof}






Using Lemma~\ref{Lemma: exact SKL two Gibbs}, we provide different characterizations of $\alpha$-weighted Gibbs algorithm and two-stage Gibbs algorithm using symmetrized KL divergence.
\subsection{$\alpha$-weighted Gibbs Algorithm}\label{app: second representation alpha weighted}
\begin{reptheorem}{Theorem: rep in skl divergence transfer}\textbf{(restated)}The expected transfer generalization error of the $\alpha$-weighted Gibbs algorithm in~\eqref{equ:Gibbs-alpha-weighted} is given by:
\begin{align}
    & \overline{\text{gen}}_{\alpha}(P_{D_s},P_{D_t}) = \frac{D_{\mathrm{SKL}}(P_{W_\alpha|D_s,D_t}^{\gamma}\|P_{W_\alpha|D_s}^{\gamma,L_\alpha(w_\alpha,d_s,P_{D_t})}|P_{D_s} P_{D_t})}{\gamma \alpha},
\end{align}
where $P_{W_\alpha|D_s}^{\gamma,L_\alpha(w_\alpha,d_s,P_{D_t})}$ is the $(\gamma,\pi(w_\alpha),L_\alpha(w_\alpha,d_s,P_{D_t}))$-Gibbs algorithm with  $L_\alpha(w,d_s,P_{D_t})\triangleq \alpha L_P(w_\alpha,P_{D_t})+(1-\alpha)L_E(w_\alpha,d_s)$.
\end{reptheorem}
\begin{proof}
Applying Lemma~\ref{Lemma: exact SKL two Gibbs} to the $\alpha$-weighted Gibbs algorithm and $(\gamma,\pi(w_\alpha),L_\alpha(w,d_s,P_{D_t}))$-Gibbs algorithm gives
\begin{align}
    &\frac{D_{\mathrm{SKL}}(P_{W_\alpha|D_s=d_s,D_t=d_t}^{\gamma}\|P_{W_\alpha|D_s=d_s}^{\gamma,L_\alpha(w_\alpha,d_s,P_{D_t})})}{\gamma}\\\nn&
    =\mathbb{E}_{\Delta\big(P_{W_\alpha|D_s=d_s,D_t=d_t}^{\gamma},P_{W_\alpha|D_s=d_s}^{\gamma,L_\alpha(w_\alpha,d_s,P_{D_t})}\big)}\left[L_\alpha(W_\alpha,d_s,P_{D_t})-L_E(W_\alpha,d_s,d_t)\right]\\\nn&
    =\alpha \mathbb{E}_{\Delta\big(P_{W_\alpha|D_s=d_s,D_t=d_t}^{\gamma},P_{W_\alpha|D_s=d_s}^{\gamma,L_\alpha(w_\alpha,d_s,P_{D_t})}\big)}\left[L_P(W_\alpha,P_{D_t})-L_E(W_\alpha,d_t)\right].
\end{align}
Notice the fact that
\begin{align*}
    \mathbb{E}_{P_{W_\alpha|D_s=d_s}^{\gamma,L_\alpha(w_\alpha,d_s,P_{D_t})}}[L_P(W_\alpha,P_{D_t})]=\mathbb{E}_{P_{D_t}}\big[\mathbb{E}_{P_{W_\alpha|D_s=d_s}^{\gamma,L_\alpha(w_\alpha,d_s,P_{D_t})}}[L_E(W_\alpha,D_t)]\big],
\end{align*}
and taking expectation over $D_s$ and $D_t$, we have
\begin{align*}
    D_{\mathrm{SKL}}(P_{W_\alpha|D_s,D_t}^{\gamma}\|P_{W_\alpha|D_s}^{\gamma,L_\alpha(w_\alpha,d_s,P_{D_t})}|P_{D_s} P_{D_t})&= \mathbb{E}_{P_{D_s}P_{D_t}} [D_{\mathrm{SKL}}(P_{W_\alpha|d_s,d_t}^{\gamma}\|P_{W_\alpha|d_s}^{\gamma,L_\alpha(w_\alpha,d_s,P_{D_t})})], \\
    & = \gamma \alpha \overline{\text{gen}}_{\alpha}(P_{D_s},P_{D_t}). \qedhere
\end{align*}
\end{proof}

In the following, we provide an explanation for the existence of two different characterizations of the  expected transfer generalization error, i.e., Theorem~\ref{Theorem: rep in skl divergence transfer} and Theorem~\ref{Theorem: Gibbs alpha Transfer result}.

For an arbitrary conditional distribution on hypothesis space $Q_{W_\alpha|D_s}$, we can write
\begin{align}
    &I(W_\alpha;D_t|D_s)=D(P_{W_\alpha,D_t|D_s}\|Q_{W_\alpha|D_s}\otimes P_{D_t}|P_{D_s})-D(P_{W_\alpha|D_s}\|Q_{W_\alpha|D_s}|P_{D_s}),\\
    &L(W_\alpha;D_t|D_s)=\mathbb{E}_{P_{D_s}}\big[\mathbb{E}_{P_{D_t}\otimes P_{W_\alpha|D_s}}[\log(Q_{W_\alpha|D_s}/P_{W_\alpha|D_t,D_s})]\big]+D(P_{W_\alpha|D_s}\|Q_{W_\alpha|D_s}|P_{D_s}).
\end{align}
Thus, the  symmetrized KL information can be written as
\begin{align}\label{eq: rep 1}
    I_{\mathrm{SKL}}(W_\alpha;D_t|D_s)&=I(W_\alpha;D_t|D_s)+L(W_\alpha;D_t|D_s)\nn\\
    &= D(P_{W_\alpha,D_t|D_s}\|Q_{W_\alpha|D_s}\otimes P_{D_t}|P_{D_s})+ \mathbb{E}_{P_{D_s}}\big[\mathbb{E}_{P_{D_t}\otimes P_{W_\alpha|D_s}}[\log(Q_{W_\alpha|D_s}/P_{W_\alpha|D_t,D_s})]\big],
\end{align}
which holds for all $Q_{W_\alpha|D_s}$. We compare this expression with the following representation:
\begin{align}
    D(P_{W_\alpha,D_t|D_s}\|Q_{W_\alpha|D_s}\otimes P_{D_t}|P_{D_s}) + D(Q_{W_\alpha|D_s} \otimes P_{D_t} \| P_{W_\alpha,D_t|D_s}|P_{D_s}).
\end{align}
The difference between these two expressions is as follows:
\begin{align}
    &I_{\mathrm{SKL}}(W_\alpha;D_t|D_s)-\left(D(P_{W_\alpha,D_t|D_s}\|Q_{W_\alpha|D_s}\otimes P_{D_t}|P_{D_s}) + D(Q_{W_\alpha|D_s} \otimes P_{D_t} \| P_{W_\alpha,D_t|D_s}|P_{D_s})\right) \nn \\
    &=\mathbb{E}_{P_{D_s}}\big[\mathbb{E}_{P_{D_t}\otimes P_{W_\alpha|D_s}}[\log(Q_{W_\alpha|D_s}/P_{W|D_t,D_s})]\big]-D(Q_{W_\alpha|D_s} \otimes P_{D_t} \| P_{W_\alpha,D_t|D_s}|P_{D_s}) \nn \\
    &=\mathbb{E}_{P_{D_s}}\big[\mathbb{E}_{P_{D_t}\otimes P_{W_\alpha|D_s}}[\log(Q_{W_\alpha|D_s}/P_{W_\alpha|D_t,D_s})]-\mathbb{E}_{P_{D_t}\otimes Q_{W_\alpha|D_s}}[\log(Q_{W_\alpha|D_s}/P_{W_\alpha|D_t,D_s})]\big]\nn \\
    &=\mathbb{E}_{P_{D_s}}\big[\mathbb{E}_{\Delta(P_{W_\alpha|D_s},Q_{W_\alpha|D_s})}[\mathbb{E}_{P_{D_t}}[\log(Q_{W_\alpha|D_s}/P_{W_\alpha|D_t,D_s})]]\big].
\end{align}
Thus, if  $Q_{W_\alpha|D_s}$ satisfies the following condition
\begin{equation}\label{Eq: rep main condition}
 \mathbb{E}_{\Delta(P_{W_\alpha|D_s}-Q_{W_\alpha|D_s})}[\mathbb{E}_{P_{D_t}}[\log(Q_{W_\alpha|D_s}/P_{W_\alpha|D_t,D_s})]]=0,
\end{equation}
then we have
\begin{align}
    I_{\mathrm{SKL}}(W_\alpha;D_t)=D(P_{W_\alpha,D_t|D_s}\|Q_{W_\alpha|D_s}\otimes P_{D_t}|P_{D_s}) + D(Q_{W_\alpha|D_s} \otimes P_{D_t} \| P_{W_\alpha,D_t|D_s}|P_{D_s}).
\end{align}
Now, if we set $(\gamma,\pi(w),L_E(w,d_s,d_t))$-Gibbs algorithm as $P_{W_\alpha|D_t,D_s}$, then it can be verified that using $(\gamma,\pi(w),L_\alpha(w_\alpha,d_s,P_{D_t}))$-Gibbs algorithm as $Q_{W_\alpha|D_s}$ would satisfy the condition in \eqref{Eq: rep main condition}. Thus, we can represent the expected transfer generalization error using both symmetrized KL information and divergence.

\subsection{Two-stage Gibbs Algorithm}
\begin{reptheorem}{Theorem: rep in skl divergence transfer two-stage}\textbf{(restated)}
The expected transfer generalization error of the two-stage Gibbs algorithm in~\eqref{Eq: two-stage Gibbs algorithm} is given by:
\begin{align*}
   &\overline{\text{gen}}_{\beta}(P_{D_s},P_{D_t})=\quad \frac{D_{\mathrm{SKL}}(P^\gamma_{W_c^t|D_t,W_\phi}\| P_{W_c^t|W_\phi}^{\gamma,L_P(w_\phi,w_c^t,P_{D_t})}|P_{D_t}P_{W_\phi})}{\gamma},
\end{align*}
where $P_{W_c^t|W_\phi}^{\gamma,L_P(w_\phi,w_c^t,P_{D_t})}$ is the $(\gamma,\pi(w_c^t),L_P(w_\phi,w_c^t,P_{D_t}))$-Gibbs algorithm.
\end{reptheorem}
\begin{proof}
Applying Lemma~\ref{Lemma: exact SKL two Gibbs} to the two-stage Gibbs algorithm and $(\gamma,\pi(w_c^t),L_P(w_\phi,w_c^t,P_{D_t}))$-Gibbs algorithm, we have
\begin{align}
    &\frac{D_{\mathrm{SKL}}(P^\gamma_{W_c^t|D_t=d_t,W_\phi=w_\phi}\| P_{W_c^t|W_\phi=w_\phi}^{\gamma,L_P(w_\phi,W_c^t,P_{D_t})})}{\gamma}\\\nn&
    = \mathbb{E}_{\Delta\left(P^\gamma_{W_c^t|D_t=d_t,W_\phi=w_\phi}\!,P_{W_c^t|W_\phi=w_\phi}^{\gamma,L_P(w_\phi,w_c^t,P_{D_t})}\right)}\left[L_P( W_c^t,w_\phi, P_{D_t})-L_E(W_c^t,w_\alpha,d_t)\right].
\end{align}
Notice the fact that
\begin{align*}
    \mathbb{E}_{P_{W_c^t|W_\phi=w_\phi}^{\gamma,L_P(w_\phi,w_c^t,P_{D_t})}}[L_P( W_c^t,w_\phi, P_{D_t})]=\mathbb{E}_{P_{D_t}}\big[\mathbb{E}_{P_{W_c^t|W_\phi=w_\phi}^{\gamma,L_P(w_\phi,w_c^t,P_{D_t})}}[L_E(W_c^t,w_\phi,d_t)]\big],
\end{align*}
and taking expectation over $W_\phi$ and $D_t$ completes the proof.
\end{proof}

\section{Asymptotic Behavior of Generalization Error for Gibbs Algorithm}

\subsection{Generalization Error}\label{app:gen_limit}
\begin{repproposition}{Prop: large gamma alpha}(\textbf{restated})
If the Hessian matrices $H^*(D_s,D_t) = H^*(D_s) = H^*$ are independent of $D_s$ and $D_t$, then the generalization error of the $\alpha$-weighted-ERM algorithm is
\begin{equation*}
     \overline{\text{gen}}_{\alpha}(P_{D_t}, P_{D_s}) =\frac{\mathbb{E}_{P_{D_s\!,D_t}}[\|\hat{W}_\alpha(D_s,D_t)-\hat{W}_\alpha(D_s)\|^2_{H^*}]}{\alpha},
\end{equation*}
where the notation $\|W\|_H^2 \triangleq W^\top H W$.
\end{repproposition}
\begin{proof}
It is shown in \citep{hwang1980laplace} that if the following Hessian matrices
\begin{align}
H^*(D_s,D_t) &\triangleq \nabla^2_w L_E(w,D_s,D_t)\big|_{w = \hat{W}_\alpha(D_s,D_t)},\\
H^*(D_s) &\triangleq \nabla^2_w L_\alpha(w,D_s,P_{D_t})\big|_{w = \hat{W}_\alpha(D_s)}
\end{align}
are not singular, then, as $\gamma \to \infty$
\begin{align}
    P_{W_\alpha|D_s,D_t}^\gamma &\to \mathcal{N}(\hat{W}_\alpha(D_s,D_t), \frac{1}{\gamma}H^*(D_s,D_t)^{-1}), \nn \\
    \text{and}\quad P_{W_\alpha|D_s}^{\gamma,L_\alpha} &\to \mathcal{N}(\hat{W}_\alpha(D_s), \frac{1}{\gamma}H^*(D_s)^{-1})
\end{align}
in distribution, and we use   $P_{W_\alpha|D_s}^{\gamma,L_\alpha}$ to denote $P_{W_\alpha|D_s}^{\gamma,L_\alpha(w_\alpha,d_s,P_{D_t})}$.

Thus, the conditional symmetrized KL divergence in Theorem~\ref{Theorem: rep in skl divergence transfer} can be evaluated directly using Gaussian approximations under the assumption that $H^*(D_s,D_t) = H^*(D_s) = H^*$,
\begin{align}
    & D_{\mathrm{SKL}}(P_{W_\alpha|D_s,D_t}^{\gamma}\|P_{W_\alpha|D_s}^{\gamma,L_\alpha}|P_{D_s} P_{D_t}) \nn \\
    & = \mathbb{E}_{P_{D_t,D_s}}\Big[\mathbb{E}_{P_{W_\alpha|D_s,D_t}^{\gamma}}\big[\log \frac{P_{W_\alpha|D_s,D_t}^{\gamma}}{ P_{W_\alpha|D_s}^{\gamma,L_\alpha}}\big] -  \mathbb{E}_{ P_{W_\alpha|D_s}^{\gamma,L_\alpha}}\big[\log \frac{P_{W_\alpha|D_s,D_t}^{\gamma}}{ P_{W_\alpha|D_s}^{\gamma,L_\alpha}}\big] \Big]\nn \\
    & = \mathbb{E}_{P_{D_t,D_s}}\Big[\mathbb{E}_{\Delta(P_{W_\alpha|D_s,D_t}^{\gamma},  P_{W_\alpha|D_s}^{\gamma,L_\alpha})} \big[ -\frac{\gamma}{2}(W_\alpha - \hat{W}_\alpha(D_s,D_t))^\top H^* (W_\alpha - \hat{W}_\alpha(D_s,D_t)) \nn\\ &\qquad\qquad+\frac{\gamma}{2}(W_\alpha - \hat{W}_\alpha(D_s))^\top H^* (W_\alpha - \hat{W}_\alpha(D_s))\big]\Big] \nn\\
    & = \gamma \mathbb{E}_{P_{D_t,D_s}}\Big[\mathbb{E}_{\Delta(P_{W_\alpha|D_s,D_t}^{\gamma},  P_{W_\alpha|D_s}^{\gamma,L_\alpha})} \big[  W_\alpha ^\top H^*  \hat{W}_\alpha(D_s,D_t)  - W_\alpha ^\top H^*  \hat{W}_\alpha(D_s) \big] \Big]\nn \\
    & = \gamma \mathbb{E}_{P_{D_t,D_s}} \big[  \hat{W}_\alpha(D_s,D_t) ^\top H^*  \hat{W}_\alpha(D_s,D_t)  - \hat{W}_\alpha(D_s,D_t) ^\top H^*  \hat{W}_\alpha(D_s)\nn \\
    &  \qquad \qquad
    - \hat{W}_\alpha(D_s) ^\top H^*  \hat{W}_\alpha(D_s,D_t)  - \hat{W}_\alpha(D_s) ^\top H^*  \hat{W}_\alpha(D_s) \big] \nn \\
    & = \gamma \mathbb{E}_{P_{D_t,D_s}} \big[  (\hat{W}_\alpha(D_s,D_t)-\hat{W}_\alpha(D_s)) ^\top H^*  (\hat{W}_\alpha(D_s,D_t)-\hat{W}_\alpha(D_s)) \big].
\end{align}
Thus,
\begin{equation*}
    \overline{\text{gen}}_{\alpha}(P_{D_s},P_{D_t}) = \frac{D_{\mathrm{SKL}}(P_{W_\alpha|D_s,D_t}^{\gamma}\| P_{W_\alpha|D_s}^{\gamma,L_\alpha}|P_{D_s} P_{D_t})}{\gamma \alpha} = \frac{\mathbb{E}_{P_{D_s\!,D_t}}[\|\hat{W}_\alpha(D_s,D_t)-\hat{W}_\alpha(D_s)\|^2_{H^*}]}{\alpha}. \qedhere
\end{equation*}
\end{proof}

\begin{repproposition}{Prop: large gamma beta}(\textbf{restated})
If Hessian matrices $H_c^*(D_t, W_{\phi}) = H_c^*(W_{\phi}) = H_c^*$ are independent of $D_s, D_t$, then the generalization error of the two-stage-ERM algorithm is
\begin{align*}
     \overline{\text{gen}}_{\beta}(P_{D_t},P_{D_s})  =\mathbb{E}_{D_s,D_t,W_\phi}[\|\hat{W}^{t}_{c}(D_t, W_{\phi})-\hat{W}^{t}_{c}(W_{\phi})\|^2_{H_c^*}].
\end{align*}
\end{repproposition}
\begin{proof}
It is shown in \citep{hwang1980laplace} that if the following Hessian matrices
\begin{align}
H_c^*(D_t,\!W_{\phi}) &\triangleq \nabla^2_{w_c}\! L_E^{S2}(W_{\phi},\!w_c,\!D_t)\big|_{w_c = \hat{W}^{t}_{c}(\!D_t,\! W_{\phi}\!)}\\
H_c^*(W_{\phi}) &\triangleq \nabla^2_{w_c}\! L_P(W_{\phi},w_c,P_{D_t})\big|_{w_c = \hat{W}^{t}_{c}( W_{\phi})}
\end{align}
are not singular, then, as $\gamma \to \infty$
\begin{align}
    P^\gamma_{W_c^t|D_t,W_\phi} &\to \mathcal{N}(\hat{W}^{t}_{c}(D_t, W_{\phi}), \frac{1}{\gamma}H_c^*(D_t,\!W_{\phi})^{-1}), \nn \\
    \quad P_{W_c^t|W_\phi}^{\gamma,L_P} &\to \mathcal{N}(\hat{W}^{t}_{c}( W_{\phi}), \frac{1}{\gamma}H^*(D_s)^{-1}),
\end{align}
where we use $P_{W_c^t|W_\phi}^{\gamma,L_P}$ to denote $P_{W_c^t|W_\phi}^{\gamma,L_P(w_\phi,w_c^t,P_{D_t})}$. Thus, the conditional symmetrized KL divergence in Theorem~\ref{Theorem: rep in skl divergence transfer two-stage} can be evaluated directly using Gaussian approximations under the assumption that $H_c^*(D_t, W_{\phi}) = H_c^*(W_{\phi}) = H_c^*$.
\begin{align}
    & D_{\mathrm{SKL}}(P^\gamma_{W_c^t|D_t,W_\phi}\| P_{W_c^t|W_\phi}^{\gamma,L_P}|P_{D_t}P_{W_\phi}) \nn \\
    & = \mathbb{E}_{P_{D_t,W_\phi}}\Big[\mathbb{E}_{P^\gamma_{W_c^t|D_t,W_\phi}}\big[\log \frac{P^\gamma_{W_c^t|D_t,W_\phi}}{P_{W_c^t|W_\phi}^{\gamma,L_P}}\big] -  \mathbb{E}_{P_{W_c^t|W_\phi}^{\gamma,L_P}}\big[\log \frac{P^\gamma_{W_c^t|D_t,W_\phi}}{P_{W_c^t|W_\phi}^{\gamma,L_P}}\big] \Big]\nn \\
    & = \mathbb{E}_{P_{D_t,W_\phi}}\Big[\mathbb{E}_{\Delta(P^\gamma_{W_c^t|D_t,W_\phi}, P_{W_c^t|W_\phi}^{\gamma,L_P})} \big[ -\frac{\gamma}{2}(W_c^t - \hat{W}^{t}_{c}(D_t, W_{\phi}))^\top H_c^* (W_c^t - \hat{W}^{t}_{c}(D_t, W_{\phi})) \nn\\ &\qquad\qquad+\frac{\gamma}{2}(W_c^t - \hat{W}^{t}_{c}( W_{\phi}))^\top H_c^* (W_c^t - \hat{W}^{t}_{c}( W_{\phi}))\big]\Big] \nn\\
    & = \gamma \mathbb{E}_{P_{D_t,W_\phi}}\Big[\mathbb{E}_{\Delta(P^\gamma_{W_c^t|D_t,W_\phi}, P_{W_c^t|W_\phi}^{\gamma,L_P})} \big[  (W_c^t) ^\top H_c^*  \hat{W}^{t}_{c}(D_t, W_{\phi})  - (W_c^t)^\top H_c^*  \hat{W}^{t}_{c}( W_{\phi}) \big] \Big]\nn \\
    & = \gamma \mathbb{E}_{P_{D_t,W_\phi}} \big[  \hat{W}^{t}_{c}(D_t, W_{\phi}) ^\top H_c^*  \hat{W}^{t}_{c}(D_t, W_{\phi})  - \hat{W}^{t}_{c}(D_t, W_{\phi}) ^\top H_c^*  \hat{W}^{t}_{c}( W_{\phi})\nn \\
    &  \qquad \qquad
    - \hat{W}^{t}_{c}( W_{\phi}) ^\top H_c^*  \hat{W}^{t}_{c}(D_t, W_{\phi})  - \hat{W}^{t}_{c}( W_{\phi}) ^\top H_c^*  \hat{W}^{t}_{c}( W_{\phi}) \big] \nn \\
    & = \gamma \mathbb{E}_{P_{D_t,W_\phi}} \big[  (\hat{W}^{t}_{c}(D_t, W_{\phi})-\hat{W}^{t}_{c}( W_{\phi})) ^\top H_c^*  (\hat{W}^{t}_{c}(D_t, W_{\phi})-\hat{W}^{t}_{c}( W_{\phi})) \big].
\end{align}
Thus,
\begin{equation*}
    \overline{\text{gen}}_{\beta}(P_{D_t},P_{D_s}) = \frac{D_{\mathrm{SKL}}(P^\gamma_{W_c^t|D_t,W_\phi}\| P_{W_c^t|W_\phi}^{\gamma,L_P}|P_{D_t}P_{W_\phi}) }{\gamma} =\mathbb{E}_{D_t,W_\phi}[\|\hat{W}^{t}_{c}(D_t, W_{\phi})-\hat{W}^{t}_{c}(W_{\phi})\|^2_{H_c^*}]. \qedhere
\end{equation*}
\end{proof}

\subsection{Regularity Conditions for MLE}\label{app:MLE}

In this section, we present the regularity conditions required by the asymptotic normality \citep{van2000asymptotic} of maximum likelihood estimates.

\begin{assumption}\label{assump:MLE}
\textbf{Regularity Conditions for MLE:}
\begin{enumerate}
  \item $f(z|\vw) \ne f(z|\vw')$ for $\vw \ne \vw'$.
  \item $\mathcal{W}$ is an open subset of $\mathbb{R}^d$.
  \item The function $\log f(z|\vw)$ is three times continuously differentiable with
respect to $\vw$.
  \item There exist functions $F_1(z): \mathcal{Z} \to \mathbb{R}$, $F_2(z): \mathcal{Z} \to \mathbb{R}$ and $M(z): \mathcal{Z} \to \mathbb{R}$, such that
\begin{equation*}
  \mathbb{E}_{Z\sim f(z|\vw)}[M(Z)] <\infty,
\end{equation*}
and the following inequalities hold for any $\vw \in \mathcal{W}$,
\begin{align*}
  \left|\frac{\partial \log f(z|\vw)}{\partial w_i} \right|<F_1(z), &\qquad  \left|\frac{\partial^2 \log f(z|\vw)}{\partial w_i \partial w_j} \right|<F_1(z), \\
  \left|\frac{\partial^3 \log f(z|\vw)}{\partial w_i \partial w_j \partial w_k} \right|<M(z), &\qquad i,j,k =1,2,\cdots,d.
\end{align*}
  \item The following inequality holds for an arbitrary $\vw \in \mathcal{W}$,
\begin{equation*}
  0< \mathbb{E}_{Z\sim f(z|\vw)}\left[\frac{\partial \log f(z|\vw)}{\partial w_i}\frac{\partial \log f(z|\vw)}{\partial w_j}\right] <\infty, \quad i,j=1,2,\cdots,d.
\end{equation*}
\end{enumerate}
\end{assumption}

\subsection{Generalization error in MLE} \label{app:MLE_limit}
\textbf{$\alpha$-weighted ERM:}
We use the following notations to denote the expectation of the Hessian matrices and the Fisher information matrices,
\begin{align*}
J_s(\vw_\alpha) \triangleq \E_{P_Z^s} \big[ -\nabla_{\vw_\alpha}^2 \log f(Z|\vw_\alpha) \big],&\quad J_t(\vw_\alpha) \triangleq \E_{P_Z^t} \big[ -\nabla_{\vw_\alpha}^2 \log f(Z|\vw_\alpha) \big],\\
    \mathcal{I}_s(\vw_\alpha) \triangleq \E_{P_Z^s} \big[ \nabla_{\vw_\alpha} \log f(Z|\vw_\alpha) \nabla_{\vw_\alpha} \log f(Z|\vw_\alpha) ^\top\big], &\quad \mathcal{I}_t(\vw_\alpha) \triangleq \E_{P_Z^t} \big[ \nabla_{\vw_\alpha} \log f(Z|\vw_\alpha) \nabla_{\vw_\alpha} \log f(Z|\vw_\alpha) ^\top\big], \\ \bar{J}(\vw_\alpha) = \frac{n}{m+n}J_s(\vw_\alpha)+\frac{m}{m+n}J_t(\vw_\alpha),&\quad
    \bar{\mathcal{I}}(\vw_\alpha) = \frac{n}{m+n}\mathcal{I}_s(\vw_\alpha)+\frac{m}{m+n}\mathcal{I}_t(\vw_\alpha).
\end{align*}

\begin{lemma}\label{Lemma: ge for alpha} Under Assumption~\ref{assump:MLE}, for any fixed source samples $d_s$, if we let $m\to \infty$, then the $\alpha$-weighted ERM satisfies
 \begin{equation}
     \sqrt{m}\big(\hat{W}_\alpha(d_s,D_t)-\hat{W}_\alpha(d_s)\big) \to \mathcal{N}\big(0,  \alpha^2\widetilde{J}(\hat{W}_\alpha(d_s))^{-1}\mathcal{I}_t(\hat{W}_\alpha(d_s)) \widetilde{J}(\hat{W}_\alpha(d_s))^{-1}\big),
 \end{equation}
 where $\widetilde{J}(\hat{W}_\alpha(d_s)) \triangleq \alpha J_t(\hat{W}_\alpha(d_s)) +(1-\alpha) \nabla^2_w L_E(w,d_s)\big|_{w = \hat{W}_\alpha(d_s)}$, and $\mathcal{I}_t(\hat{W}_\alpha(d_s))$ is the covariance matrix of $\nabla_w \log f(Z^t|\hat{W}_\alpha(d_s))$.
\end{lemma}

\begin{proof}
By using a Taylor expansion of the first derivative of
the weighted log-likelihood $L_E(\hat{W}_\alpha(d_s,D_t),d_s,D_t)$ around $\hat{W}_\alpha(d_s)$, we obtain
\begin{align}
    0 &= \nabla_w L_E(w,d_s,D_t)\big|_{w = \hat{W}_\alpha(d_s,D_t)} \\
    &\approx \nabla_w L_E(w,d_s,D_t)\big|_{w = \hat{W}_\alpha(d_s)} + \nabla^2_w L_E(w,d_s,D_t)\big|_{w = \hat{W}_\alpha(d_s)} (\hat{W}_\alpha(d_s,D_t)-\hat{W}_\alpha(d_s)). \nn
\end{align}
From the Taylor series expansion formula, the following approximation can be obtained
\begin{equation}\label{equ:proof_approx}
    -\nabla^2_w L_E(w,d_s,D_t)\big|_{w = \hat{W}_\alpha(d_s)} (\hat{W}_\alpha(d_s,D_t)-\hat{W}_\alpha(d_s)) \approx \nabla_w L_E(w,d_s,D_t)\big|_{w = \hat{W}_\alpha(d_s)}.
\end{equation}
By the law of large numbers, when $m \to \infty$, it can be shown that
\begin{equation}\label{equ:proof_J}
    -\nabla^2_w L_E(\hat{W}_\alpha(d_s),D_t) = \frac{1}{m}\sum_{i=1}^m\nabla^2_w \log f( Z_i^t|\hat{W}_\alpha(d_s)) \to -J_t(\hat{W}_\alpha(d_s)).
\end{equation}
Thus, the LHS of \eqref{equ:proof_approx} can be written as
\begin{equation}
    \nabla^2_w L_E(w,d_s,D_t)\big|_{w = \hat{W}_\alpha(d_s)} =  \nabla^2_w \big[\alpha L_E(w,D_t) +(1-\alpha) L_E(w,d_s)\big]\big|_{w = \hat{W}_\alpha(d_s)}\to \widetilde{J}(\hat{W}_\alpha(d_s)),
\end{equation}
where $\widetilde{J}(\hat{W}_\alpha(d_s)) = \alpha J_t(\hat{W}_\alpha(d_s)) +(1-\alpha) \nabla^2_w L_E(w,d_s)\big|_{w = \hat{W}_\alpha(d_s)}$.

As for the RHS of \eqref{equ:proof_approx}, note that
\begin{align}
  \sqrt{m} \nabla_w L_E(w,D_t)\big|_{w = \hat{W}_\alpha(d_s)} = -\frac{1}{\sqrt{m}} \sum_{i=1}^m\nabla_w \log f(Z_i^t|\hat{W}_\alpha(d_s)),
\end{align}
 by multivariate central limit theorem
\begin{align}
  \frac{1}{\sqrt{m}} \sum_{i=1}^n\Big(-\nabla_w \log f(Z_i^t|\hat{W}_\alpha(d_s))+\mathbb{E}_{Z^t}[\nabla_w \log f(Z^t|\hat{W}_\alpha(d_s))]\Big) \to \mathcal{N}(0, \mathcal{I}_t(\hat{W}_\alpha(d_s))),
\end{align}
where $\mathcal{I}_t(\hat{W}_\alpha(d_s)) $ is the covariance matrix of $\nabla_w \log f(Z^t|\hat{W}_\alpha(d_s))$.


Due to the definition of $\hat{W}_\alpha(d_s)$, we have $\nabla_w L_E(w,d_s,P_{D_t})\big|_{w = \hat{W}_\alpha(d_s)}=0$, i.e.,
\begin{equation}
   (1-\alpha) \nabla_w L_E(\hat{W}_\alpha(d_s),d_s)  =\alpha \mathbb{E}_{Z^t}[\nabla_w \log f(Z^t|\hat{W}_\alpha(d_s))].
\end{equation}
Thus, the RHS of \eqref{equ:proof_approx} will converge to
 \begin{equation}
     \sqrt{m} \nabla_w L_E(w,D_s,D_t)\big|_{w = \hat{W}_\alpha(D_s)} \to \mathcal{N}\big(0,\alpha^2 \mathcal{I}_t(\hat{W}_\alpha(d_s))  \big).
 \end{equation}
Combining with \eqref{equ:proof_J}, when $m \to \infty$, we obtain
 \begin{equation}
     \sqrt{m}\big(\hat{W}_\alpha(d_s,D_t)-\hat{W}_\alpha(d_s)\big) \to \mathcal{N}\big(0,  \alpha^2\widetilde{J}(\hat{W}_\alpha(d_s))^{-1}\mathcal{I}_t(\hat{W}_\alpha(d_s)) \widetilde{J}(\hat{W}_\alpha(d_s))^{-1}\big).
 \end{equation}
\end{proof}
In the main body of the paper, we further let $n \to \infty$, then $\hat{W}_\alpha(d_s) \to \vw^*_\alpha$, and $\widetilde{J}(\hat{W}_\alpha(d_s)) \to \bar{J}(\vw^*_\alpha)$, $\mathcal{I}_t(\hat{W}_\alpha(d_s)) \to \mathcal{I}_t(\vw^*_\alpha)$. For $\alpha = \frac{m}{m+n}$, using Lemma~\ref{Lemma: ge for alpha}, we can show that
\begin{equation}
     \hat{W}_\alpha(D_s,D_t)-\hat{W}_\alpha(D_s) \to \mathcal{N}\big(0,  \frac{m}{(m+n)^2}\bar{J}(\vw^*_\alpha)^{-1}\mathcal{I}_t(\vw^*_\alpha) \bar{J}(\vw^*_\alpha)^{-1}\big).
 \end{equation}
In addition, the Hessian matrix $H^*(D_s,D_t) \to \bar{J}(\vw^*_{\alpha})$ as $m, n\to \infty$, which is independent of the samples $D_s,D_t$. Proposition~\ref{Prop: large gamma alpha} gives
\begin{equation*}
    \overline{\text{gen}}_{\alpha}(P_{D_t},P_{D_s})  = \frac{\mathrm{tr}( \mathcal{I}_t(\vw^*_\alpha)\bar{J}(\vw^*_{\alpha})^{-1})}{n+m}=\mathcal{O}(\frac{d}{m+n}).
\end{equation*}

\textbf{Two-stage ERM:}

We use the following notations to denote the expectation of the Hessian matrix and the Fisher information matrix with respect to $\vw_c$,
\begin{align*}
J_c^t({\vw}_{\phi},{\vw}_{c}) &\triangleq \E_{P_Z^t} \big[-\nabla^2_{\vw_c} \log f( Z|[{\vw}_{\phi},{\vw}_{c}]) \big],\\
    \mathcal{I}^t_c({\vw}_{\phi},{\vw}_{c})
    &\triangleq \mathbb{E}_{P_Z^t}[\nabla_{\vw_c} \log f( Z|[{\vw}_{\phi},{\vw}_{c})])\nabla_{\vw_c}^\top \log  f( Z|[{\vw}_{\phi},{\vw}_{c}])].
\end{align*}

\begin{lemma}\label{Lemma: ge for two} Under Assumption~\ref{assump:MLE}, for any fixed  $\hat{\vw}_{\phi}$, if we let $m\to \infty$, then the two-stage ERM satisfies
\begin{equation}
     \sqrt{m}\big((\hat{W}^{t}_{c}(D_t, \hat{\vw}_{\phi})-\hat{\vw}^{t}_{c}(\hat{\vw}_{\phi}))\big) \to \mathcal{N}\big(0, J_c^t(\hat{\vw}_{\phi},\hat{\vw}^{t}_{c}(\hat{\vw}_{\phi})) ^{-1}\mathcal{I}^t_c(\hat{\vw}_{\phi},\hat{\vw}^{t}_{c}(\hat{\vw}_{\phi}))J_c^t(\hat{\vw}_{\phi},\hat{\vw}^{t}_{c}(\hat{\vw}_{\phi})) ^{-1} \big).
 \end{equation}

\end{lemma}

\begin{proof}
For any fixed $\hat{\vw}_{\phi}$, using a Taylor expansion of the gradient with respect to $\vw_c$ of
the log-likelihood $L_E^{S2}(\hat{\vw}_{\phi},\hat{W}^{t}_{c}(D_t, \hat{\vw}_{\phi}),D_t)$ around $\hat{\vw}^{t}_{c}(\hat{\vw}_{\phi})$, we obtain
\begin{align*}
    0 &= \nabla_{\vw_c} L_E^{S2}(\hat{\vw}_{\phi},\hat{W}^{t}_{c}(D_t, \hat{\vw}_{\phi}),D_t) \\
    &\approx \nabla_{\vw_c} L_E^{S2}(\hat{\vw}_{\phi},\hat{\vw}^{t}_{c}(\hat{\vw}_{\phi}),D_t)
    + \nabla^2_{\vw_c} L_E^{S2}(\hat{\vw}_{\phi},\hat{\vw}^{t}_{c}(\hat{\vw}_{\phi}),D_t) (\hat{W}^{t}_{c}(D_t, \hat{\vw}_{\phi})-\hat{\vw}^{t}_{c}(\hat{\vw}_{\phi})).
\end{align*}
From the Taylor series expansion formula, the following approximation can be obtained
\begin{equation}\label{equ:proof_two_stage}
    -\nabla^2_{\vw_c} L_E^{S2}(\hat{\vw}_{\phi},\hat{\vw}^{t}_{c}(\hat{\vw}_{\phi}),D_t) (\hat{W}^{t}_{c}(D_t, \hat{\vw}_{\phi})-\hat{\vw}^{t}_{c}(\hat{\vw}_{\phi})) \approx \nabla_{\vw_c} L_E^{S2}(\hat{\vw}_{\phi},\hat{\vw}^{t}_{c}(\hat{\vw}_{\phi}),D_t).
\end{equation}
By the law of large numbers, when $m \to \infty$, it can be shown that
\begin{align}
    &-\nabla^2_{\vw_c} L_E^{S2}(\hat{\vw}_{\phi},\hat{\vw}^{t}_{c}(\hat{\vw}_{\phi}),D_t) = \frac{1}{m}\sum_{i=1}^m\nabla^2_{\vw_c} \log f( Z_i^t|[\hat{\vw}_{\phi},\hat{\vw}^{t}_{c}(\hat{\vw}_{\phi})]) \to -J_c^t(\hat{\vw}_{\phi},\hat{\vw}^{t}_{c}(\hat{\vw}_{\phi})).
\end{align}
As for the RHS of \eqref{equ:proof_two_stage}, note that $\mathbb{E}_{P_Z^t}[\nabla_{\vw_c} \log f( Z|[\hat{\vw}_{\phi},\hat{\vw}^{t}_{c}(\hat{\vw}_{\phi})])]=0$ due to the definition of $\hat{\vw}^{t}_{c}(\hat{\vw}_{\phi})$, by multivariate central limit theorem, we have
\begin{align}
  &\frac{1}{\sqrt{m}} \sum_{i=1}^n\Big(-\nabla_{\vw_c} \log f( Z_i^t|[\hat{\vw}_{\phi},\hat{\vw}^{t}_{c}(\hat{\vw}_{\phi})])\Big) \to \mathcal{N}(0, \mathcal{I}^t_c(\hat{\vw}_{\phi},\hat{\vw}^{t}_{c}(\hat{\vw}_{\phi})) ),
\end{align}
where $\mathcal{I}^t_c(\hat{\vw}_{\phi},\hat{\vw}^{t}_{c}(\hat{\vw}_{\phi}))= \mathbb{E}_{P_Z^t}[\nabla_{\vw_c} \log f( Z|[\hat{\vw}_{\phi},\hat{\vw}^{t}_{c}(\hat{\vw}_{\phi})])\nabla_{\vw_c}^\top \log f(Z|[\hat{\vw}_{\phi},\hat{\vw}^{t}_{c}(\hat{\vw}_{\phi})])]$.

Thus, the RHS of \eqref{equ:proof_two_stage} will converge to
 \begin{equation}
    \sqrt{m} \nabla_{\vw_c} L_E^{S2}(\hat{\vw}_{\phi},\vw_c,D_t)\big|_{\vw_c = \hat{\vw}^{t}_{c}(\hat{\vw}_{\phi})} \to \mathcal{N}\big(0,\mathcal{I}^t_c(\hat{\vw}_{\phi},\hat{\vw}^{t}_{c}(\hat{\vw}_{\phi}))  \big).
 \end{equation}
When $m \to \infty$, we obtain
 \begin{equation}
     \sqrt{m}\big((\hat{W}^{t}_{c}(D_t, \hat{\vw}_{\phi})-\hat{\vw}^{t}_{c}(\hat{\vw}_{\phi}))\big) \to \mathcal{N}\big(0, J_c^t(\hat{\vw}_{\phi},\hat{\vw}^{t}_{c}(\hat{\vw}_{\phi})) ^{-1}\mathcal{I}^t_c(\hat{\vw}_{\phi},\hat{\vw}^{t}_{c}(\hat{\vw}_{\phi}))J_c^t(\hat{\vw}_{\phi},\hat{\vw}^{t}_{c}(\hat{\vw}_{\phi})) ^{-1} \big).
 \end{equation}
\end{proof}

In the main body of the paper, we further let $n \to \infty$, then $\hat{\vw}_{\phi} \to \vw^{s*}_\phi$, and $\hat{\vw}^{t}_{c}(\hat{\vw}_{\phi}) \to \vw^{st*}_c$. Using Lemma~\ref{Lemma: ge for two}, we can show that
\begin{align*}
    \hat{W}^{t}_{c}(D_t, \hat{W}_{\phi}) - \hat{W}^{t}_{c}(\hat{W}_{\phi}) \rightarrow
    \mathcal{N}\big(0, \frac{J_c^t(\vw^{s*}_{\phi},\vw^{st*}_{c})^{-1}\mathcal{I}^t_c(\vw^{s*}_{\phi},\vw^{st*}_{c}) J_c^t(\vw^{s*}_{\phi},\vw^{st*}_{c}) ^{-1}}{m}\big).
\end{align*}
As the Hessian matrix $H_c^*(D_t, W_{\phi}) = H_c^*(W_{\phi}) \to J_c^t(\vw^{s*}_{\phi},\vw^{st*}_{c}) $ as $m, n\to \infty$. By Proposition~\ref{Prop: large gamma beta}, we have
\begin{equation}
    \overline{\text{gen}}_{\beta}(P_{D_t},P_{D_s})  = \frac{\mathrm{tr}\big(\mathcal{I}^t_c(\vw^{s*}_{\phi},\vw^{st*}_{c}) J_c^t(\vw^{s*}_{\phi},\vw^{st*}_{c}) ^{-1}\big)}{m} = \mathcal{O}(\frac{d_c}{m}).
\end{equation}

\subsection{Excess risk} \label{app:excess_limit}

\textbf{$\alpha$-weighted ERM:} In the following lemma, we characterize the variance of the $\alpha$-weighted ERM algorithm.

\begin{lemma} Under Assumption~\ref{assump:MLE}, if we let $m,n\to \infty$, then the $\alpha$-weighted ERM satisfies
 \begin{equation}
     \sqrt{m+n}\big(\hat{W}_\alpha(D_s,D_t)-\vw^*_{\alpha}\big) \to \mathcal{N}\big(0,  \bar{J}(\vw^*_{\alpha})^{-1}\bar{\mathcal{I}}_t(\vw^*_{\alpha}) \bar{J}(\vw^*_{\alpha})^{-1}\big).
 \end{equation}
\end{lemma}

\begin{proof}
By using a Taylor expansion of the first derivative of
the weighted log-likelihood $L_E(\hat{W}_\alpha(D_s,D_t),D_s,D_t)$ around $\vw^*_{\alpha}$, we obtain
\begin{equation*}
    0 = \nabla_w L_E(w,D_s,D_t)\big|_{w = \hat{W}_\alpha(D_s,D_t)} \approx \nabla_w L_E(w,D_s,D_t)\big|_{w =\vw^*_{\alpha}} + \nabla^2_w L_E(w,D_s,D_t)\big|_{w = \vw^*_{\alpha}} (\hat{W}_\alpha(D_s,D_t)-\vw^*_{\alpha}).
\end{equation*}
From the Taylor series expansion formula, the following approximation can be obtained
\begin{equation}\label{equ:proof_approx_excess}
    -\nabla^2_w L_E(w,D_s,D_t)\big|_{w = \vw^*_{\alpha}} (\hat{W}_\alpha(D_s,D_t)-\vw^*_{\alpha}) \approx \nabla_w L_E(w,D_s,D_t)\big|_{w = \vw^*_{\alpha}}.
\end{equation}
By the law of large numbers, when $m,n \to \infty$, it can be shown that
\begin{align}
    &-\nabla^2_w L_E(\vw^*_{\alpha},D_t) = \frac{1}{m}\sum_{i=1}^m\nabla^2_w \log f( Z_i^t|\vw^*_{\alpha}) \to -J_t(\vw^*_{\alpha}), \\
     &-\nabla^2_w L_E(\vw^*_{\alpha},D_s) = \frac{1}{n}\sum_{i=1}^n\nabla^2_w \log f( Z_i^s|\vw^*_{\alpha}) \to -J_s(\vw^*_{\alpha}).
\end{align}
Thus, the LHS of \eqref{equ:proof_approx_excess} converges to
\begin{equation}
    \nabla^2_w L_E(w,D_s,D_t)\big|_{w = \vw^*_{\alpha}} \to \bar{J}(\vw^*_{\alpha}),
\end{equation}
where $\bar{J}(\vw^*_{\alpha}) \triangleq \alpha J_t(\vw^*_{\alpha}) +(1-\alpha) J_t(\vw^*_{\alpha})$.

As for the RHS of \eqref{equ:proof_approx_excess}, by multivariate central limit theorem
\begin{align}
  &\frac{1}{\sqrt{m}} \sum_{i=1}^n\Big(-\nabla_w \log f(Z_i^t|\vw^*_{\alpha})+\mathbb{E}_{Z^t}[\nabla_w \log f(Z^t|\vw^*_{\alpha})]\Big) \to \mathcal{N}(0, \mathcal{I}_t(\vw^*_{\alpha})), \\
  &\frac{1}{\sqrt{n}} \sum_{i=1}^n\Big(-\nabla_w \log f(Z_i^s|\vw^*_{\alpha})+\mathbb{E}_{Z^s}[\nabla_w \log f(Z^s|\vw^*_{\alpha})]\Big) \to \mathcal{N}(0, \mathcal{I}_s(\vw^*_{\alpha})),
\end{align}
where $\mathcal{I}_t(\vw^*_{\alpha})$ and $\mathcal{I}_s(\vw^*_{\alpha})$ are the covariance matrix of $\nabla_w \log f(Z^t|\vw^*_{\alpha})$ and $\nabla_w \log f(Z^s|\vw^*_{\alpha})$, respectively.

Due to the definition of $\vw^*_{\alpha}$, we have
\begin{equation}
   (1-\alpha) \mathbb{E}_{Z^s}[\nabla_w \log f(Z^s|\vw^*_{\alpha})] +\alpha \mathbb{E}_{Z^t}[\nabla_w \log f(Z^t|\vw^*_{\alpha})] =0.
\end{equation}
Thus, the RHS of \eqref{equ:proof_approx_excess} will converge to
 \begin{equation}
    \nabla_w L_E(w,D_s,D_t)\big|_{w = \vw^*_{\alpha}} \to \mathcal{N}\big(0,\frac{\alpha^2}{m} \mathcal{I}_t(\vw^*_{\alpha}) +\frac{(1-\alpha)^2}{n} \mathcal{I}_s(\vw^*_{\alpha})  \big).
 \end{equation}
When $m,n \to \infty$, we obtain
 \begin{equation}
     \big(\hat{W}_\alpha(D_s,D_t)-\vw^*_{\alpha}\big) \to \mathcal{N}\big(0,  \bar{J}(\vw^*_{\alpha})^{-1}(\frac{\alpha^2}{m} \mathcal{I}_t(\vw^*_{\alpha}) +\frac{(1-\alpha)^2}{n} \mathcal{I}_s(\vw^*_{\alpha})) \bar{J}(\vw^*_{\alpha})^{-1}\big).
 \end{equation}
For $\alpha = \frac{m}{m+n}$, if we denote $\bar{\mathcal{I}}(\vw_\alpha) = \frac{n}{m+n}\mathcal{I}_s(\vw_\alpha)+\frac{m}{m+n}\mathcal{I}_t(\vw_\alpha)$, we have
 \begin{equation}
     \big(\hat{W}_\alpha(D_s,D_t)-\vw^*_{\alpha}\big) \to \mathcal{N}\big(0,  \frac{1}{m+n}\bar{J}(\vw^*_{\alpha})^{-1} \bar{\mathcal{I}}(\vw^*_{\alpha}) \bar{J}(\vw^*_{\alpha})^{-1}\big). \qedhere
 \end{equation}
\end{proof}
Thus, the variance term in the excess risk can be computed as:
\begin{equation}
    \mathrm{tr}(J_t(\vw_t^*) \Cov(\hat{W}_\alpha(D_s,D_t))) =\frac{ \mathrm{tr}(J_t(\vw_t^*) \bar{J}(\vw^*_{\alpha})^{-1} \bar{\mathcal{I}}(\vw^*_{\alpha}) \bar{J}(\vw^*_{\alpha})^{-1})}{m+n}= \mathcal{O}(\frac{d}{m+n}).
\end{equation}


\textbf{Two-stage ERM:}
We use the following notations to denote the expectation of the Hessian matrix and the Fisher information matrix with respect to $\vw_\phi$,
\begin{align*}
J_{c,\phi}^t(\vw_{\phi},\vw_{c}) &\triangleq \E_{P_Z^t} \big[-\nabla^2_{\vw_c,\vw_\phi} \log f( Z|[\vw_{\phi},\vw_{c}]) \big],\\
J^s_\phi(\vw_{\phi})&\triangleq \E_{P_Z^s} \big[-\nabla^2_{\vw_\phi} \log f( Z|[\vw_{\phi},\vw_{c}]) \big],\\
    \mathcal{I}^s_\phi({\vw}_{\phi},{\vw}_{c})
    &\triangleq \mathbb{E}_{P_Z^s}[\nabla_{\vw_\phi} \log f( Z|[{\vw}_{\phi},{\vw}_{c})])\nabla_{\vw_\phi}^\top \log  f( Z|[{\vw}_{\phi},{\vw}_{c}])].
\end{align*}

In the following lemma, we characterize the variance of the two-stage ERM algorithm.

\begin{lemma} Under Assumption~\ref{assump:MLE}, if we let $m,n \to \infty$, then the two-stage ERM satisfies
\begin{align}
     &\big(\hat{W}^{t}_{c}(\hat{W}_{\phi}, D_t)-\vw^{st*}_{c}\big) \to\mathcal{N}\Big(0,  J_c^t(\vw^{s*}_{\phi},\vw^{st*}_{c})^{-1}\\
     & \quad \big(\frac{1}{m}\mathcal{I}^t_c(\vw^{s*}_{\phi},\vw^{st*}_{c}) + \frac{1}{n} J_{c,\phi}^t(\vw^{s*}_{\phi},\vw^{st*}_{c})J^s_\phi(\vw^{s*}_{\phi})^{-1}\mathcal{I}^s_\phi(\vw^{s*}_{\phi})J^s_\phi(\vw^{s*}_{\phi})^{-1} J_{c,\phi}^t(\vw^{s*}_{\phi},\vw^{st*}_{c})\big)  J_c^t(\vw^{s*}_{\phi},\vw^{st*}_{c}) ^{-1}\Big).\nn
 \end{align}
\end{lemma}

\begin{proof}
By using a Taylor expansion of the gradient with respect to $\vw_c$ of
$L_E^{S2}(\hat{W}_{\phi}(D_s),\hat{W}^{t}_{c}(\hat{W}_{\phi}, D_t),D_t)$ around $[\vw^{s*}_{\phi}, \vw^{st*}_{c}]$, we obtain
\begin{align*}
    0 &= \nabla_{\vw_c} L_E^{S2}(\hat{W}_{\phi}(D_s),\hat{W}^{t}_{c}(\hat{W}_{\phi}, D_t),D_t) \\
    &\approx \nabla_{\vw_c} L_E^{S2}(\vw^{s*}_{\phi},\vw^{st*}_{c},D_t)  + \nabla^2_{\vw_c,\vw_\phi} L_E^{S2}(\vw^{s*}_{\phi},\vw^{st*}_{c},D_t) (\hat{W}_{\phi}(D_s)-\vw^{s*}_{\phi}) \nn \\
    &\quad + \nabla^2_{\vw_c} L_E^{S2}(\vw^{s*}_{\phi},\vw^{st*}_{c},D_t) (\hat{W}^{t}_{c}(\hat{W}_{\phi}, D_t)-\vw^{st*}_{c}).
\end{align*}

From the Taylor series expansion formula, the following approximation can be obtained
\begin{align}\label{equ:proof_two_stage_excess}
    &-\nabla^2_{\vw_c} L_E^{S2}(\vw^{s*}_{\phi},\vw^{st*}_{c},D_t) (\hat{W}^{t}_{c}(\hat{W}_{\phi}, D_t)-\vw^{st*}_{c}) \nn \\
    &\approx \nabla_{\vw_c} L_E^{S2}(\vw^{s*}_{\phi},\vw^{st*}_{c},D_t) + \nabla^2_{\vw_c,\vw_\phi} L_E^{S2}(\vw^{s*}_{\phi},\vw^{st*}_{c},D_t) (\hat{W}_{\phi}(D_s)-\vw^{s*}_{\phi}).
\end{align}
By the law of large numbers, when $m \to \infty$, it can be shown that
\begin{align}
    &-\nabla^2_{\vw_c} L_E^{S2}(\vw^{s*}_{\phi},\vw^{st*}_{c},D_t) = \frac{1}{m}\sum_{i=1}^m\nabla^2_{\vw_c}\log f( Z_i^t|[\vw^{s*}_{\phi},\vw^{st*}_{c}]) \to -J_c^t(\vw^{s*}_{\phi},\vw^{st*}_{c}), \\
     &-\nabla^2_{\vw_c,\vw_\phi} L_E^{S2}(\vw^{s*}_{\phi},\vw^{st*}_{c},D_t) = \frac{1}{m}\sum_{i=1}^m\nabla^2_{\vw_c,\vw_\phi} \log f( Z_i^t|[\vw^{s*}_{\phi},\vw^{st*}_{c}]) \to -J_{c,\phi}^t(\vw^{s*}_{\phi},\vw^{st*}_{c}).
\end{align}
As for the first term in the RHS of \eqref{equ:proof_two_stage_excess}, note that $\mathbb{E}_{P_Z^t}[\nabla_{\vw_c} \log f( Z|[\vw^{s*}_{\phi},\vw^{st*}_{c}])]=0$, by multivariate central limit theorem, we have
\begin{align}
  &\frac{1}{\sqrt{m}} \sum_{i=1}^n\Big(-\nabla_{\vw_c} \log f( Z_i^t|[\vw^{s*}_{\phi},\vw^{st*}_{c}])\Big) \to \mathcal{N}(0, \mathcal{I}^t_c(\vw^{s*}_{\phi},\vw^{st*}_{c})).
\end{align}
When $n\to \infty$, due to the asymptotic normality of maximum likelihood estimate, we have
\begin{equation}
    \sqrt{n}(\hat{W}_{\phi}(D_s)-\vw^{s*}_{\phi}) \to \mathcal{N}(0, J^s_\phi(\vw^{s*}_{\phi})^{-1} \mathcal{I}^s_\phi(\vw^{s*}_{\phi})J^s_\phi(\vw^{s*}_{\phi})^{-1}),
\end{equation}
where  $\mathcal{I}^s_\phi(\vw^{s*}_{\phi}) =  \mathbb{E}_{P_Z^s}[\nabla_{\vw_\phi} \log f( Z|[\vw^{s*}_{\phi},\vw^{s*}_{c} ])\nabla_{\vw_\phi}^\top \log f( Z|[\vw^{s*}_{\phi},\vw^{s*}_{c} ])]$.

Thus, the RHS of \eqref{equ:proof_two_stage_excess} converges to
\begin{equation}
\mathcal{N}\Big(0,\  \frac{1}{m}\mathcal{I}^t_c(\vw^{s*}_{\phi},\vw^{st*}_{c}) + \frac{1}{n} J_{c,\phi}^t(\vw^{s*}_{\phi},\vw^{st*}_{c})J^s_\phi(\vw^{s*}_{\phi})^{-1}\mathcal{I}^s_\phi(\vw^{s*}_{\phi})J^s_\phi(\vw^{s*}_{\phi})^{-1} J_{c,\phi}^t(\vw^{s*}_{\phi},\vw^{st*}_{c})\Big)
\end{equation}
when $m,n \to \infty$.

Thus, we obtain
 \begin{align}
     &\big(\hat{W}^{t}_{c}(\hat{W}_{\phi}, D_t)-\vw^{st*}_{c}\big) \to\mathcal{N}\Big(0,  J_c^t(\vw^{s*}_{\phi},\vw^{st*}_{c})^{-1}\\
     & \quad \big(\frac{1}{m}\mathcal{I}^t_c(\vw^{s*}_{\phi},\vw^{st*}_{c}) + \frac{1}{n} J_{c,\phi}^t(\vw^{s*}_{\phi},\vw^{st*}_{c})J^s_\phi(\vw^{s*}_{\phi})^{-1}\mathcal{I}^s_\phi(\vw^{s*}_{\phi})J^s_\phi(\vw^{s*}_{\phi})^{-1} J_{c,\phi}^t(\vw^{s*}_{\phi},\vw^{st*}_{c})\big)  J_c^t(\vw^{s*}_{\phi},\vw^{st*}_{c}) ^{-1}\Big).\nn
 \end{align}
\end{proof}
Note that $\Cov(\hat{W}_{\phi}(D_s))$ can be characterized by the asymptotic normality of maximum likelihood estimate. Thus, the variance term in the excess risk can be computed as:
\begin{equation}
   \mathrm{tr}\big(J_t(\vw^{t*}_{\phi}, \vw^{t*}_{c}) \Cov(\hat{W}_{\phi}(D_s),\hat{W}^{t}_{c}(D_t, \hat{W}_{\phi}))\big) = \mathcal{O}(\frac{d_c}{m}+\frac{d}{n}).
\end{equation}

\end{document}